\documentclass[12pt]{article}
\pdfoutput=1
\usepackage{fullpage}
\usepackage{authblk}
\usepackage{graphicx}
\usepackage{amsmath}
\usepackage{amssymb}
\usepackage{tikz}
\usetikzlibrary{arrows.meta}
\usepackage{xcolor}
\usepackage{algorithm}
\usepackage[noEnd=true]{algpseudocodex}
\usepackage[multiple]{footmisc}
\usepackage{hyperref}

\newcommand\Vtextvisiblespace[1][.3em]{%
  \mbox{\kern.06em\vrule height.3ex}%
  \vbox{\hrule width#1}%
  \hbox{\vrule height.3ex}}
\newcommand{\blank}{\Vtextvisiblespace[0.4em]}

\newcommand{\NN}{\mathbb{N}}

\newtheorem{theorem}{Theorem}
\newtheorem{lemma}[theorem]{Lemma}
\newtheorem{proposition}[theorem]{Proposition}

\newtheorem{corollary}[theorem]{Corollary}

\newenvironment{proof}{\begin{trivlist}
\item[\hspace{\labelsep}{\em\noindent Proof: }]
}{\hspace{\stretch{1}}\rule{1.5ex}{1.5ex}\end{trivlist}}

\title{Autoregressive Large Language Models \\ are Computationally Universal}
\author[1,2]{Dale Schuurmans}
\author[1]{Hanjun Dai}
\author[2]{Francesco Zanini}
\affil[1]{Google DeepMind}
\affil[2]{University of Alberta}
\date{}

\begin{document}

\maketitle

\begin{abstract}

We show that autoregressive decoding of a transformer-based
language model can realize universal computation,
without external intervention or modification of the model's weights.
Establishing this result requires understanding how a language model
can process arbitrarily long inputs using a bounded context.
For this purpose, we consider a generalization of autoregressive decoding
where, given a long input, emitted tokens are appended to the end of the
sequence as the context window advances.
We first show that the resulting system corresponds to a classical model
of computation, a \emph{Lag system},
that has long been known to be computationally universal.
By leveraging a new proof, we show that a universal Turing machine 
can be simulated by a Lag system with $2027$ production rules.
We then investigate whether an existing large language model
can simulate the behaviour of such a universal Lag system.
We give an affirmative answer by showing that a single system-prompt can be
developed for {\tt gemini-1.5-pro-001} that drives the model,
under deterministic (greedy) decoding, to correctly apply
each of the $2027$ production rules.
We conclude that, by the Church-Turing thesis, prompted
{\tt gemini-1.5-pro-001}
with extended autoregressive (greedy) decoding is a general purpose computer.

\end{abstract}

\section{Introduction}
\label{sec:intro}

The emergence of large language models has raised 
fundamental questions about their computational capability
relative to classical models of computation.
Several works have investigated the computational abilities
of large language models,
for example, by considering the expressiveness of
transformer architectures for representing circuits 
\cite{perezetal19,bhattamishraetal20,weietal22a}
and for representing sequential versions of such circuits
\cite{fengetal23,merrillsabharwal24}
under bounded chain of thought extensions \cite{weietal22b}.
In this paper, we consider the more general question of whether 
a large language model can support universal computation
when applying unbounded chain of thought.
Recently, it has been established that a large language model can be
augmented with an external memory to enable simulation of a universal
Turing machine via prompting \cite{schuurmans23}.
However, such a result is weakened by the use of external control mechanisms---%
in particular, regular expression parsers---%
that offload computational responsibility from the language model.
The question of whether an unaided large language model can be
Turing universal has not previously been resolved.

We provide an affirmative answer by showing
that an unaided large language model can simulate a universal Turing machine.
Achieving such a result requires a more general view of autoregressive decoding
that allows processing of arbitrarily long input strings.
In particular, we consider a natural generalization
of autoregressive decoding
where emitted tokens are appended to the end of the sequence
after processing each successive context,
which reduces to standard autoregressive decoding whenever the 
input fits within the context window.

The main result in this paper is established in a series of steps.
First, in Section~\ref{sec:auto}, we introduce the more
general perspective on
autoregressive decoding that accommodates long input strings,
subsequently
showing in Section~\ref{sec:lag} that the proposed extension drives a language
model to realize a restricted form of \emph{Lag system} \cite{wang63}---%
a variant of one of the earliest general models of computation \cite{post43}.
We then show in Section~\ref{sec:queue} that a Lag system,
which is able to organize memory as a circular queue,
can also provide bidirectional control over memory access.
After relevant background on finite memory simulation of Turing machines
in Section~\ref{sec:tm},
Section~\ref{sec:sim} then proves that 
any Turing machine can be simulated by
a restricted Lag system with a context length of $2$.
Although the universality of Lag systems has been known since \cite{wang63},
the alternative proof presented in this paper is more direct and enables the
subsequent argument.
Next, in Section~\ref{sec:utm}, we apply the reduction technique to a specific
universal Turing machine, $U_{15,2}$ \cite{nearywoods09},
and obtain a universal Lag system that is defined by a set of $2027$ production 
rules over an alphabet of $262$ symbols.
Finally, Section~\ref{sec:gemini} develops a single system-prompt that is able 
to drive a particular large language model, 
{\tt gemini-1.5-pro-001},
to correctly apply each of the $2027$ rules under greedy decoding.
This outcome allows us to conclude that 
{\tt gemini-1.5-pro-001}
with extended
autoregressive (greedy) decoding can exactly simulate the execution of 
$U_{15,2}$ on any input, hence it is a general purpose computer.

\section{Autoregressive decoding}
\label{sec:auto}

Given a finite alphabet $\Sigma=\{\sigma_1,...,\sigma_n\}$,
a \emph{string} is defined to be a finite sequence of symbols $s_1...s_k$
such that $k\in\NN$ and $s_i\in\Sigma$ for $1\leq i\leq k$.
Note that every string has a finite length but there is no upper bound
on the length of a string.
We let $|s|$ denote the length of a string $s$,
let $\Sigma^*$ denote the set of all strings,
and let $H\subset\Sigma$ denote a set of halt symbols.

A language model expresses a conditional distribution $p(s_{n+1}|s_1...s_n)$
over a next symbol $s_{n+1}\in\Sigma$ given an input string $s_1...s_n$. 
Any such model can be extended to a conditional distribution over an output
sequence $s_{n+1}...s_{n+k}$ by the chain rule of probability
\begin{eqnarray}
p(s_{n+1}...s_{n+k}|s_1...s_n) & = &
p(s_{n+1}|s_1...s_n)\,p(s_{n+2}|s_1...s_{n+1})\cdots p(s_{n+k}|s_1...s_{n+k-1})
.
\label{eq:chain rule}
\end{eqnarray}
Therefore, an output sequence $s_{n+1}...s_{n+k}$ can be generated from the
conditional distribution on the left hand side 
of \eqref{eq:chain rule}
by sampling each successive
symbol from the corresponding conditional distribution on the right hand side,
i.e.,
$s_{n+1}\sim p(\cdot|s_1...s_n)$,
...,
$s_{n+k}\sim p(\cdot|s_1...s_{n+k-1})$.
This process is referred to as \emph{autoregressive decoding}.
In practice, decoding typically proceeds until
a halt symbol in $H$ is generated or a maximum emission length is reached.

Note that a key restriction of any transformer based language model is that
it employs a bounded context that limits the length of the input string 
to at most $N$ symbols for some $N<\infty$.
However, to capture any universal notion of computation
we will need to consider computing over arbitrarily long inputs,
while also allowing for the possibility of arbitrarily long output sequences.
Therefore, we require
a more general perspective on autoregressive decoding
that is able to accommodate long input and output sequences.

First, if we assume the input is shorter than the context,
a long output sequence can be generated simply by adopting
the standard $N$-Markov assumption
\begin{eqnarray}
p(s_{n+k}|s_1...s_{n+k-1}) & = &
\left\{
\begin{array}{ll}
p(s_{n+k}|s_{n+k-N}...s_{n+k-1}) & \mbox{ if } n+k>N+1
\\
p(s_{n+k}|s_1...s_{n+k-1}) & \mbox{ if } n+k\leq N+1
\end{array}
\right.
\label{eq:N markov}
\end{eqnarray}
for all $s_1...s_{n+k-1}$;
that is, the conditioning can be truncated to the most recent $N$ symbols.
Under this assumption,
autoregressive decoding is able to generate output sequences by
successively sampling $s_{n+k}\sim p(\cdot|s_{n+k-N}...s_{n+k-1})$
for arbitrarily large $k$,
which is a default assumption in the literature.

\begin{figure}
\centering
\hskip2cm
\begin{tikzpicture}[baseline=1]
\draw[-{Latex},red] (4.030,1.189) .. controls (4.833,1.4) and (5.637,1.4) .. (6.440,1.189);
\put (0,0) {
$s_1...s_k s_{k+1}\fcolorbox{green}{white}{$s_{k+2}...s_{k+N}s_{k+N+1}$}...s_{n+k}s_{n+k+1}\fcolorbox{red}{white}{$s_{n+k+2}$}$ };
\draw[-{Latex},red] (5.300,0.289) .. controls (6.067,0.5) and (6.833,0.5) .. (7.600,0.289);
\put (0,25) {
$s_1...s_k\fcolorbox{green}{white}{$s_{k+1}s_{k+2}...s_{k+N}$}s_{k+N+1}...s_{n+k}\fcolorbox{red}{white}{$s_{n+k+1}$}$ };
\end{tikzpicture}
\caption{
\emph{Generalized autoregressive decoding}
when the length of the input sequence $n$
exceeds the context length $N$.
Here $k$ is the number of output symbols that have already been appended.
The figure depicts the generation of the $(k+1)$st output symbol
conditioning on the context of length $N$
starting at index $k+1$,
followed by the generation of the $(k+2)$nd output symbol
conditioning on the context of length $N$
starting at index $k+2$.
This process reduces to standard $N$-Markov autoregressive decoding when
$n\leq N$.
}
\label{fig:gen auto}
\end{figure}
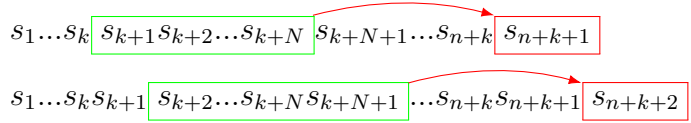

The key question remains of how to perform autoregressive decoding
when the input is longer than the context length.
Such an extension is not typically considered in the literature,
but necessary in our case.
For this purpose, 
we consider a simple generalization of
$N$-Markov autoregressive decoding, \emph{generalized autoregressive decoding},
where the next generated symbol is appended to the end of the sequence,
as shown in Figure~\ref{fig:gen auto}.
In particular, given an input of length $n$,
if we start from an initial context window at index $1$ and
let $n+k$ denote the length of the entire sequence after $k$ symbols
have been emitted, 
then assuming $n>N$, $k+1$ will index the start of the current context
and the $(k+1)$st output symbol will be generated
according to $s_{n+k+1}\sim p(\cdot|s_{k+1}...s_{k+N})$;
see Figure~\ref{fig:gen auto}.
Note that if $n\leq N$ the process reduces to standard $N$-Markov 
autoregressive decoding \eqref{eq:N markov}.
Although such an autoregressive decoding mechanism might seem peculiar,
given that the language model was trained to predict $s_{k+N+1}$ 
rather than $s_{n+k+1}$ from the context $s_{k+1}...s_{k+N}$,
we will see that this generalization
allows
the language model to exhibit rich computational behaviour.

To compare the computational ability of a
language model to classical models of computation,
we furthermore need to assume the conditional distribution is
deterministic;
that is, there exists a function $M:\Sigma^N\rightarrow\Sigma$ 
such that, for all strings $s_1...s_N$,
$p(\sigma|s_1...s_N)=1$ if $\sigma=M(s_1...s_N)$ and
$p(\sigma|s_1...s_N)=0$ if $\sigma\neq M(s_1...s_N)$.
Such deterministic behaviour can be achieved by setting the
temperature parameter of the language model to zero and fixing any additional
random seeds affecting its output.
Thus, the $(k+1)$st output symbol
is generated according to $s_{n+k+1}=M(s_{k+1}...s_{k+N})$
then appended to the end of the sequence.

Finally,
to accommodate a special edge case that we will have to elaborate below,
the language model will be required to emit a pair of output symbols 
rather than just a single symbol in certain contexts,
as shown in Figure~\ref{fig:gen auto 2}.
In this case, we use $\ell$ to denote the number of output symbols that have
already been appended after $k$ iterations, with $\ell\geq k$.
Note that
such a mechanism does not conform to the generalized autoregressive decoding 
mechanism depicted in Figure~\ref{fig:gen auto}, 
because $s_{n+\ell+2}$ will need to depend on $s_{k+1}$ and not
$s_{k+N+1}$, even after $s_{n+\ell+1}$ is generated,
as shown in Figure~\ref{fig:gen auto 2}.
Therefore, we need to consider another form of decoding,
\emph{extended autoregressive decoding},
to achieve such behaviour.
To realize this extension
we leverage the existence of
an implicit halt token $h$ outside the base alphabet $\Sigma$
(disjoint from the explicit halt tokens $H$),
and decode the language model given a context $s_{\ell+1}...s_{\ell+N}$
by generating symbols from $\Sigma\cup\{h\}$ until
$h$ is encountered (then discarded).
This technicality will be revisited when it becomes relevant,
but briefly it is the differentiator between being able to simulate
a general Turing machine versus a linear bounded automaton.

\begin{figure}
\centering
\hskip2cm
\begin{tikzpicture}[baseline=1]
\draw[-{Latex},red] (4.030,0.289) .. controls (4.820,0.5) and (5.610,0.5) .. (6.400,0.289);
\put (0,0) {
$s_1...s_k\fcolorbox{green}{white}{$s_{k+1}s_{k+2}...s_{k+N}$}s_{k+N+1}...s_{n+\ell}\fcolorbox{red}{white}{$s_{n+\ell+1}s_{n+\ell+2}$}$ };
\end{tikzpicture}
\caption{
\emph{Extended autoregressive decoding} when two output symbols 
are generated for a given context.
Here $\ell$ is the number of output symbols that have already been appended
after $k$ iterations, with $\ell\geq k$, and it is assumed the length of the
input sequence $n$ exceeds the context length $N$.
}
\label{fig:gen auto 2}
\end{figure}
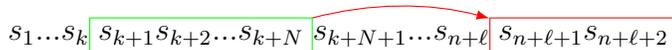

In summary, we consider deterministic autoregressive decoding of a language
model with context length $N$
according to Algorithm~\ref{alg:auto}.%
\footnote{
Note that the language model does not need to maintain the count $\ell$
of the number of symbols appended;
we include it simply to keep an end-of-sequence index $n+\ell$
that will be helpful to refer to in some proofs.
}

\begin{algorithm}
\caption{Extended autoregressive decoding of model $M$ with context length $N$}
\label{alg:auto}
\textbf{Input:} $s\leftarrow s_1...s_n$
\begin{algorithmic}
\State $\ell\gets0$
\For{k=0,1,...}
\State $c\gets s_{k+1}...s_{k+N}$ \Comment{get next context}
\State $y\gets M(c)$ \Comment{get model response}
\State $s\gets s_1...s_{n+\ell}\,y$\Comment{append response to overall sequence}
\State $\ell\gets\ell+|y|$ \Comment{track number of symbols appended}
\If{$y\cap H\neq\emptyset$}
\State halt
\EndIf
\EndFor
\end{algorithmic}
\end{algorithm}

\section{Lag systems}
\label{sec:lag}

The first key observation is that autoregressive decoding
of a large language model can be replicated by a Lag system.
Lag systems were introduced by \cite{wang63} as a simple variation of 
one of the earliest formal models of general computation,
the Tag systems studied in \cite{post43}.

A \emph{Lag system} consists of a finite set of rules $x_1...x_N\rightarrow y$,
where $N$ is the length of the context,
$x_1...x_N\in\Sigma^N$ 
denotes a sequence of symbols to be matched,
and $y\in\Sigma^*$ is a corresponding output.
For a deterministic Lag system, each pattern $x_1...x_N$ is unique,
hence the Lag system defines a partial function $L:\Sigma^N\rightarrow\Sigma^*$
that maps a pattern $x_1...x_N$ to a corresponding output $y$.
The computation of a Lag system is defined by operating over a memory string,
where in each iteration, 
a rule is matched to the prefix of the memory string,
then
the result appended to the string,
after which the first symbol is deleted; see Algorithm~\ref{alg:lag}.%
\footnote{
Once again,
the Lag system does not need to maintain the count $\ell$;
we include it simply to keep an explicit end-of-sequence index $n+\ell$.
}

\begin{algorithm}
\caption{Lag system $L$ with context length $N$}
\label{alg:lag}
\textbf{Input:} $s\leftarrow s_1...s_n$
\begin{algorithmic}
\State $\ell\gets0$
\For{k=0,1,...}
\If{$L$ contains a rule $x_1...x_N\rightarrow y$ such that $x_1...x_N=s_{k+1}...s_{k+N}$}
\State $s\gets s_{k+2}...s_{n+\ell}\,y$ \Comment{delete first symbol then append the response to sequence}
\Else
\State halt
\EndIf
\State $\ell\gets\ell+|y|$ \Comment{track number of symbols appended}
\If{$y\cap H\neq\emptyset$}
\State halt
\EndIf
\EndFor
\end{algorithmic}
\end{algorithm}

Like autoregressive decoding, we introduce an explicit set of
halt symbols $H\subset\Sigma$
so that computation terminates whenever a halt symbol is generated.
Computation also terminates when
there is no rule matching the current context $s_{k+1}...s_{k+N}$,
which additionally covers the case when the current memory string is shorter
than $N$.%
\footnote{
Classically, Lag systems allow the output $y$ to be empty, hence the string
$s$ being tracked in Algorithm~\ref{alg:lag} can contract.
However, we will not make use of this possibility in this paper.
}

Note that, by inspection, the computation of a Lag system
specified by Algorithm~\ref{alg:lag} is similar to that of
autoregressively decoding a language model in Algorithm~\ref{alg:auto}.
We formalize this observation in the following proposition.

\begin{proposition}
\label{prop:auto2lag}
For any deterministic language model $M$ with context length $N$,
there exists a deterministic Lag system $L$ such that,
for any input string $s$ with $|s|\geq N$,
the execution of Algorithm~\ref{alg:lag} with $L$ simulates
extended
autoregressive decoding of $M$ with Algorithm~\ref{alg:auto},
in the sense that,
for every iteration $k\in\NN$,
the two algorithms will encounter the same context $s_{k+1}...s_{k+N}$
and generate the same next output $y$.
\end{proposition}

\begin{proof}
Given a language model $M$ with context length $N$, construct a
Lag system $L$ where for every string
$x_1...x_N$ of length $N$ we add a rule
$x_1...x_N\rightarrow M(x_1...x_N)$ to $L$.

Consider an input string $s_1...s_n$ where $n\geq N$.
Let $s^{M,k}=s^{M,k}_1...s^{M,k}_{n+\ell^M}$ denote the
string maintained by $M$ at the start of iteration $k$,
and let
$s^{L,k}=s^{L,k}_{k+1}...s^{L,k}_{n+\ell^L}$ denote the
string maintained by $L$ at the start of iteration $k$.
We prove by induction that,
for all $k\in\NN$,
\begin{eqnarray}
s^{M,k}_{k+1}...s^{M,k}_{n+\ell^M}
& = &
s^{L,k}_{k+1}...s^{L,k}_{n+\ell^L}
\label{eq:match}
\end{eqnarray}
at the start of iteration $k$, 
where $\ell^M=\ell^L=\ell$ denote the total number of symbols
that have previously been appended by the respective algorithms.

For the base case, when $k=0$, we have that $s^{M,0}=s^{L,0}=s_1...s_n$,
which immediately satisfies \eqref{eq:match}.

Next, assume both algorithms have produced the respective strings
$s^{M,k}=s^{M,k}_1...s^{M,k}_{n+\ell}$ 
and
$s^{L,k}=s^{L,k}_{k+1}...s^{L,k}_{n+\ell}$ 
satisfying \eqref{eq:match}
at the start of the $k$th iteration.
Then the $N$ context symbols 
$s^{M,k}_{k+1}...s^{M,k}_{k+N}=s^{L,k}_{k+1}...s^{L,k}_{k+N}$
are identical in both algorithms,
hence $L(s^{L,k}_{k+1}...s^{L,k}_{k+N})=M(s^{M,k}_{k+1}...s^{M,k}_{k+N})$ by construction,
implying that both algorithms will generate the same $y=M(s^{M,k}_{k+1}...s^{M,k}_{k+N})$ as output.
This means that at the start of the next iteration $k+1$,
Algorithm~\ref{alg:auto} will be maintaining the string
$s^{M,k+1}=s^{M,k}_1...s^{M,k}_{n+\ell}\,y$
and 
Algorithm~\ref{alg:lag} will be maintaining the string
$s^{L,k+1}=s^{L,k}_{k+2}...s^{L,k}_{n+\ell}\,y$,
so the pair still satisfies \eqref{eq:match}.
\end{proof}

Thus, the computational capacity of a large language model under extended
autoregressive decoding (Algorithm~\ref{alg:auto}) does not exceed that of a
Lag system with the same context length $N$.
Proposition~\ref{prop:auto2lag} clearly motivates our interest in understanding
the computational expressiveness of Lag systems,
since the processing of a Lag system is closely related to 
(and ultimately simulable by)
extended autoregressive decoding of a language model.

We will find it useful to characterize subclasses of Lag systems,
\emph{restricted $(N,K)$-Lag systems},
by their context length $N$ and by the maximum length $K$ of any output $y$
in the right hand side of a rule.
Of particular interest will be restricted $(N,1)$ and $(N,2)$-Lag systems.
One of the key lemmas below will be to establish that any Turing machine
can be simulated by a restricted $(2,2)$-Lag system,
while any linear bounded automaton \cite{myhill60} can be simulated by
a restricted $(2,1)$-Lag system.

Although
Lag systems have been known to be computationally universal since \cite{wang63},
the main arguments in this paper require a more succinct construction based on 
a direct reduction.
To achieve such a reduction,
we first need to understand how a Lag system can enable sufficiently flexible
memory access even though it operates solely through Algorithm~\ref{alg:lag}.

\section{Memory access control with a Lag system}
\label{sec:queue}

Demonstrating that any Turing machine can be simulated by a Lag system
requires a series of developments.
The first key step is understanding how a Lag system can exercise sufficient 
control over memory access
to simulate the bidirectional memory access patterns of a Turing machine.
In this section, we demonstrate how bidirectional memory access control can be
implemented in a Lag system.

First, note that a Lag system can naturally operate
a circular queue over its memory string.
Consider the Lag system defined by the simple set of rules
$x\rightarrow x\;\forall x\in\Sigma$.
In this system,
for any input string $s_1s_2s_3...s_n$, 
Algorithm~\ref{alg:lag}
will produce an update sequence on successive iterations that behaves as
\begin{equation}
s_1s_2s_3...s_n \Rightarrow
s_2s_3...s_ns_1 \Rightarrow
s_3...s_ns_1s_2 \Rightarrow
\cdots
;
\label{eq:queue}
\end{equation}
that is, each iteration will rotate one symbol from the start 
(dequeue) to the end of the string (enqueue),
thus perpetually rotating the string in a circular queue.
It is known that a finite state machine that operates on
a memory queue can simulate any Turing machine
\cite{zaiontz76,kudlekrogozhin01},
but the challenge here is to achieve the same capability without
leveraging an external finite state controller,
since a Lag system (or autoregressive decoding of a language model)
only performs updates based on the current memory context
without access to any external state machine.

The remainder of this section shows how a Lag system can implement
bidirectional access control over its memory,
separating the two cases of moving a control location one position left 
(counterclockwise) and one position right (clockwise).

\subsection{Counterclockwise position control}
\label{sec:counterclockwise}

For a base alphabet $\Sigma$, consider an extended alphabet over pairs
$\Sigma_p=\Sigma\times\{\blank\,,\texttt p\}$,
where
in each pair $(\sigma,p)\in\Sigma_p$,
we use the second coordinate $p\in\{\blank\,,\texttt p\}$ 
as a position control symbol
and the first coordinate $\sigma\in\Sigma$ as the original ``data''.
We develop a Lag system that can move
a targeted control location one step left (counterclockwise).

Consider a data string $s_1...s_{n-2}s_{n-1}s_n\in\Sigma^n$ with $n\geq3$.
Imagine we formulate a corresponding string of augmented pairs,
$(s_1,\blank)...(s_{n-2},\blank)(s_{n-1},\texttt p)(s_n,\blank)\in\Sigma_p^n$,
such that a position control symbol $\texttt p$ has been associated with the
$(n-1)$st symbol and all other position control symbols are assigned ``blank''.
Assume we would like to move the position control symbol one step to the left
(counterclockwise);
that is, we would like to associate $\texttt p$ with $s_{n-2}$.
This can be achieved 
in a Lag system
by the following set of rules
\begin{eqnarray}
R_p & = &
\left\{
\begin{array}{ll}
(x,\blank)(y,\blank)\rightarrow(x,\blank) & \forall x,y\in\Sigma
\\
(x,\blank)(y,\texttt p)\rightarrow(x,\texttt p) & \forall x,y\in\Sigma
\\
(x,\texttt p)(y,\blank)\rightarrow(x,\blank) & \forall x,y\in\Sigma
\end{array}
\right\}
\quad
\begin{array}{lcl}
\blank\:\blank       & \rightarrow & \blank    \\
\blank\:\texttt p    & \rightarrow & \texttt p \\
\texttt p\:\blank    & \rightarrow & \blank    \\
\end{array}
\label{eq:p rules}
\end{eqnarray}

\begin{proposition}
Define a Lag system by $R_p$.
Consider any string
\rm\[
(s_1,\blank)...(s_{n-2},\blank)(s_{n-1},\texttt p)(s_n,\blank)
\]\em
$n\geq3$.
Then after $n-1$ iterations Algorithm~\ref{alg:lag} will render the memory
string
\rm\[
(s_n,\blank)(s_1,\blank)...(s_{n-2},\texttt p)(s_{n-1},\blank)
\]\em
thus moving the control token {\rm$\texttt p$} one step counterclockwise.
\label{prop:p}
\end{proposition}

\begin{proof}
Whenever the position control symbols in the first two pairs (the ``context'')
are both blank,
the first pair will be rotated unaltered to the end of the string,
by the rules corresponding to the pattern $\blank\,\blank\rightarrow\blank\,$.
Therefore, for the first $n-3$ iterations,
the first pair will simply be rotated to the end unmodified.
Then, on iteration $n-2$,
the context encountered will be $(s_{n-2},\blank)(s_{n-1},\texttt p)$,
implying that $(s_{n-2},\texttt p)$ will be rotated to the end
according to
$\blank\,\texttt p\rightarrow\texttt p$.
On iteration $n-1$, the next context will be $(s_{n-1},\texttt p)(s_n,\blank)$,
so $(s_{n-1},\blank)$ will be rotated to the end
by $\texttt p\,\blank\rightarrow\blank\,$,
yielding the stated result.
\end{proof}

Ideally, we would also like to distinguish an initiating control symbol
$\texttt L$ from a terminating control symbol $\texttt l$, so we can
differentiate the beginning of a move from its conclusion.
This can be achieved by a minor alteration of the rule set.
\begin{eqnarray}
R_{\textrm{left}} & = &
\left\{
\begin{array}{ll}
(x,\blank)(y,\blank)\rightarrow(x,\blank) & \forall x,y\in\Sigma
\\
(x,\blank)(y,\texttt L)\rightarrow(x,\texttt l) & \forall x,y\in\Sigma
\\
(x,\texttt L)(y,\blank)\rightarrow(x,\blank) & \forall x,y\in\Sigma
\end{array}
\right\}
\quad
\begin{array}{lcl}
\blank\:\blank       & \rightarrow & \blank    \\
\blank\:\texttt L    & \rightarrow & \texttt l \\
\texttt L\:\blank    & \rightarrow & \blank    \\
\end{array}
\label{eq:left rules}
\end{eqnarray}

\begin{proposition}
Define a Lag system by $R_{\textrm{left}}$.
Consider any string
\rm\[
(s_1,\blank)...(s_{n-2},\blank)(s_{n-1},\texttt L)(s_n,\blank)
\]\em
$n\geq3$.
Then after $n-1$ iterations Algorithm~\ref{alg:lag} will render the memory
string
\rm\[
(s_n,\blank)(s_1,\blank)...(s_{n-2},\texttt l)(s_{n-1},\blank)
\]\em
changing the control token from {\rm$\texttt L$} to {\rm$\texttt l$}
and rotating the control position one step counterclockwise.
\label{prop:left}
\end{proposition}

\begin{proof}
The first $n-3$ iterations simply rotate the first pair unmodified
to the end.
Then, on iteration $n-2$,
the encountered context will be $(s_{n-2},\blank)(s_{n-1},\texttt L)$,
implying that $(s_{n-2},\texttt l)$ will be rotated to the end
by
$\blank\,\texttt L\rightarrow\texttt l$.
On iteration $n-1$, the next context will be $(s_{n-1},\texttt L)(s_n,\blank)$,
implying that $(s_{n-1},\blank)$ will be rotated to the end
by
$\texttt L\,\blank\rightarrow\blank\,$,
rendering the stated result.
\end{proof}

Intuitively, the presence of the control token $\texttt l$ informs $s_{n-2}$
that it has become the next control point, after which appropriate
updates can be subsequently triggered.
Thus, a Lag system can easily and naturally move control tokens in a
counterclockwise direction around a circular queue,
incurring only $O(n)$ iterations to move a control token and realign the queue.

\subsection{Clockwise position control}
\label{sec:clockwise}

Moving a control location one step right (clockwise) is also achievable,
but it is far more involved.
Such functionality can be achieved by a construction that takes $O(n^3)$ 
iterations. 
It will take a series of developments to arrive at a suitable
rule set in an understandable way.

First, consider an example string
$(s_1,\blank)(s_2,\blank)...(s_{n-1},\texttt p)(s_n,\blank)$
where we would like to move the control token $\texttt p$
one step clockwise (right) rather than counterclockwise (left);
that is, we would like to associate $\texttt p$ with $s_n$.
To understand how such a clockwise move might be possible,
consider the previous rule set $R_p$
but note that if we run Algorithm~\ref{alg:lag} for a larger number of
iterations,
the pulse token $\texttt p$ will continue to move around the queue
in a counterclockwise manner until it eventually lands on $s_n$.

\begin{proposition}
For the Lag system defined by $R_p$ and any string
\rm\[
(s_1,\blank)(s_2,\blank)...(s_{n-1},\texttt p)(s_n,\blank)
\]\em
$n\geq3$,
after $(n-1)^2$ iterations
Algorithm~\ref{alg:lag} will render the memory string
\rm\[
(s_2,\blank)...(s_{n-1},\blank)(s_n,\texttt p)(s_1,\blank)
\]\em
effectively rotating the control token {\rm$\texttt p$} one step clockwise.

Continuing for another $n-1$ iterations will restore the memory string
to its initial state,
so the overall behaviour is periodic with period $n(n-1)$.
\label{prop:p squared}
\end{proposition}

\begin{proof}
By Proposition~\ref{prop:p} we know that after $n-1$ iterations
the control token $\texttt p$ will have moved one position counterclockwise
to $s_{n-2}$.
Therefore, repeating this process $n-2$ times,
i.e., $(n-2)(n-1)$ iterations in total,
the memory string will become
$(s_3,\blank)...(s_n,\blank)(s_1,\texttt p)(s_2,\blank)$.
From this string, another $n-1$ iterations will lead to
$(s_2,\blank)...(s_{n-1},\blank)(s_n,\texttt p)(s_1,\blank)$,
yielding the first stated claim after $(n-2)(n-1)+n-1=(n-1)^2$ iterations.

Continuing from this string for another $n-1$ iterations will result in the
original string, using a total of $(n-1)^2+n-1=n(n-1)$ iterations.
\end{proof}

Therefore, after every block of $n(n-1)$ iterations 
the token $\texttt p$ will complete an orbit of the memory string.
We will refer to each block of $n-1$ iterations as a \emph{pass},
and each block of $n$ passes (i.e., $n(n-1)$ iterations) as a \emph{cycle}.

The key consequence of Proposition~\ref{prop:p squared} is that,
if we had remembered $s_{n-1}$ as the initial control position,
then 
after $(n-1)^2+n-2$ iterations we would encounter the context
$(s_{n-1},\blank)(s_n,\texttt p)$,
at which point the first symbol $s_{n-1}$ would be able to ``see'' that
the pulse $\texttt p$ has reached its clockwise neighbour $s_n$,
and we could potentially notify $s_n$ that it has become the new control point.
Unfortunately, the information that $s_{n-1}$ was the original control position
has been lost, so we need to augment the construction to retain this.

To retain a memory of the initial control position, 
consider the extended alphabet
$\Sigma_t=\Sigma_p\cup(\Sigma\times\{\texttt t,\texttt w\})$
and the extended rule set
\begin{eqnarray}
R_t & = &
R_p\cup
\left\{
\begin{array}{ll}
(x,\blank)(y,\texttt t)    \rightarrow(x,\texttt p) & \forall x,y\in\Sigma \\
(x,\blank)(y,\texttt w)    \rightarrow(x,\blank)    & \forall x,y\in\Sigma \\
(x,\texttt t)(y,\blank)    \rightarrow(x,\texttt w) & \forall x,y\in\Sigma \\
(x,\texttt w)(y,\blank)    \rightarrow(x,\texttt w) & \forall x,y\in\Sigma \\
(x,\texttt w)(y,\texttt p) \rightarrow(x,\texttt t) & \forall x,y\in\Sigma \\
(x,\texttt p)(y,\texttt w) \rightarrow(x,\blank)    & \forall x,y\in\Sigma
\end{array}
\right\}
\quad
\begin{array}{lcl}
\blank\:\texttt t    & \rightarrow & \texttt p \\
\blank\:\texttt w    & \rightarrow & \blank    \\
\texttt t\:\blank    & \rightarrow & \texttt w \\
\texttt w\:\blank    & \rightarrow & \texttt w \\
\texttt w\:\texttt p & \rightarrow & \texttt t \\
\texttt p\:\texttt w & \rightarrow & \blank 
\end{array}
\label{eq:t rules}
\end{eqnarray}
Here we have introduced two new ``stationary'' tokens $\texttt t$ and
$\texttt w$ that will remain at the original location $s_{n-1}$ to remember
the initial control point,
while still allowing the pulse token $\texttt p$
to orbit the memory string unimpeded.
In particular, $\texttt w$ serves as the stationary marker at $s_{n-1}$.
The second token $\texttt t$ is needed to represent the situation
when $\texttt p$ arrives at location $s_{n-1}$,
which occurs once during each orbit.
Whenever this happens, we combine $\texttt p$ and $\texttt w$ into $\texttt t$
to remember the control location and also represent the current location
of $\texttt p$.
On the subsequent pass, $\texttt t$ will re-initiate $\texttt p$ in the
next counterclockwise position, $s_{n-2}$, then switch back to $\texttt w$.
Such a mechanism allows $\texttt p$ to continue its synchronous orbit
around the memory queue, while $\texttt w$ retains knowledge of the
initial location $s_{n-1}$.
We formalize this process in the following proposition.

\begin{proposition}
For the Lag system defined by $R_t$ and any string
\rm\[
(s_1,\blank)...(s_{n-2},\blank)(s_{n-1},\texttt t)(s_n,\blank)
\]\rm
$n\geq3$,
after $n-1$ passes ($(n-1)^2$ iterations)
Algorithm~\ref{alg:lag} will render the memory string
\rm\[
(s_2,\blank)...(s_{n-1},\texttt w)(s_n,\texttt p)(s_1,\blank)
\]\em
so that the control symbol {\rm $\texttt w$} remembers the initial position
and the pulse token {\rm $\texttt p$} has been placed one step clockwise
from the initial position of $\texttt t$.

Moreover,
after an additional pass ($n-1$ iterations) the memory string will be restored
to its initial state, yielding periodic behaviour with period $n(n-1)$
(i.e., one cycle).
\end{proposition}

\begin{proof}
The first pass behaves as follows.
Note that the first $n-3$ iterations simply rotate the first pair 
to the end unaltered.
On iteration $n-2$,
the encountered context will be $(s_{n-2},\blank)(s_{n-1},\texttt t)$,
implying that $(s_{n-2},\texttt p)$ will be rotated to the end
by the rule $\blank\,\texttt t\rightarrow\texttt p$.
On iteration $n-1$,
the encountered context will be $(s_{n-1},\texttt t)(s_n,\blank)$,
so $(s_{n-1},\texttt w)$ will be rotated to the end
by $\texttt t\,\blank\rightarrow\texttt w$.
Therefore, at the conclusion of the first pass ($n-1$ iterations)
the memory will consist of
$(s_n,\blank)(s_1,\blank)...(s_{n-2},\texttt p)(s_{n-1},\texttt w)$.

On the second pass,
the next $n-3$ iterations will rotate the first pair unaltered.
The subsequent iteration will encounter the context
$(s_{n-3},\blank)(s_{n-2},\texttt p)$
and rotate $(s_{n-3},\texttt p)$ to the end,
while the following iteration will encounter
$(s_{n-2},\texttt p)(s_{n-1},\texttt w)$
and rotate $(s_{n-2},\blank)$ to the end
by the rule $\texttt p\,\texttt w\rightarrow\blank\,$.
Therefore, at the conclusion of the second pass the memory will consist of
$(s_{n-1},\texttt w)(s_n\blank)(s_1,\blank)...(s_{n-3},\texttt p)(s_{n-2},\blank)$.

Each of the next $n-3$ passes
will result in 
the symbol $\texttt w$ remaining at $s_{n-1}$,
by the rule $\texttt w\,\blank\rightarrow\texttt w$,
while $\texttt p$ will be advanced one position counterclockwise,
as established in Proposition~\ref{prop:p squared}.
These claims follow because $\texttt p$ and $\texttt w$ will both be
followed by blanks.
Moreover, the position control symbol associated with $s_{n-2}$
remains blank because the context
$(s_{n-2},\blank)(s_{n-1},\texttt w)$ will rotate $(s_{n-2},\blank)$
to the end by the rule $\blank\,\texttt w\rightarrow\blank\,$.
Therefore, at the conclusion of $n-1$ passes the memory will consist of
$(s_2,\blank)...(s_{n-1},\texttt w)(s_n,\texttt p)(s_1,\blank)$,
which establishes the first claim after $(n-2)(n-1)+n-1=(n-1)^2$ iterations.

Continuing from the memory string
$(s_2,\blank)...(s_{n-1},\texttt w)(s_n,\texttt p)(s_1,\blank)$,
the next $n-3$ iterations will rotate the first pair unaltered,
resulting in the context $(s_{n-1},\texttt w)(s_n,\texttt p)$.
The next iteration will rotate $(s_{n-1},\texttt t)$
to the end, by the rule $\texttt w\,\texttt p\rightarrow\texttt t$,
leaving the context $(s_n,\texttt p)(s_1,\blank)$.
Finally, $(s_n,\blank)$ will be rotated to the end unaltered,
resulting in the original string
after a total of $(n-1)^2+n-1=n(n-1)$ iterations, or $n$ passes.
\end{proof}

Therefore, at the end of $(n-1)^2+n-2$ iterations, 
we will encounter the context $(s_{n-1},\texttt w)(s_n,\texttt p)$,
making it possible for the Lag system to recognize that the pulse
$\texttt p$ has arrived at the correct clockwise neighbour $s_n$.
At this point, 
to move the control locus clockwise to $s_n$, 
all we will need to do is somehow ``freeze'' an appropriate control token on 
$s_n$ while clearing the other position control tokens.
This leads to the final construction for clockwise rotation control.

To achieve the ability to freeze an orbiting token once a desired position
is detected, we introduce a two time-scale message passing scheme.
First, as above,
a fast pulse signal $\texttt p$ will orbit the memory queue
once every cycle (every $n$ passes, or $n(n-1)$ iterations),
visiting each successive counterclockwise position after each 
pass ($n-1$ iterations).
We also retain the same stationary tokens $\texttt w$ and $\texttt t$
to remember the initial control position in the same manner as above.
On top of these mechanisms, we add another,
slower pulsed-hold signal $\texttt d$
that also propagates counterclockwise around the queue,
but only advances when ``pulsed'' by $\texttt p$.
That is, $\texttt d$ will only advance one position counterclockwise for every
full cycle completed by $\texttt p$.
(This is analogous to a clock, where $\texttt p$ behaves like a second hand
and $\texttt d$ like a minute hand, both rotating counterclockwise.)
When the pulsed-hold symbol $\texttt d$ finally reaches the target location
$s_n$, it will be visible from the initial control location $s_{n-1}$.
At that point, a final ``victory'' pulse $\texttt v$ will be sent around the
queue to inform $s_n$, which is still holding the control token $\texttt d$, 
that it is indeed the next control location.

This scheme is realized by the extended alphabet
$\Sigma_{\textrm{right}}=
\Sigma_t\cup
(\Sigma\times\{\texttt R,\texttt d,\texttt g,\texttt z,\texttt v,\texttt r\})$
and the extended rule set
\begin{eqnarray}
R_{\textrm{right}} & = &
R_t\cup
\left\{
\begin{array}{ll}
(x,\blank)(y,\texttt R)    \rightarrow(x,\texttt d) & \forall x,y\in\Sigma \\
(x,\blank)(y,\texttt d)    \rightarrow(x,\blank)    & \forall x,y\in\Sigma \\
(x,\blank)(y,\texttt g)    \rightarrow(x,\texttt z) & \forall x,y\in\Sigma \\
(x,\blank)(y,\texttt z)    \rightarrow(x,\texttt p) & \forall x,y\in\Sigma \\
(x,\blank)(y,\texttt v)    \rightarrow(x,\texttt v) & \forall x,y\in\Sigma \\
(x,\blank)(y,\texttt r)    \rightarrow(x,\blank)    & \forall x,y\in\Sigma \\
(x,\texttt R)(y,\blank)    \rightarrow(x,\texttt t) & \forall x,y\in\Sigma \\
(x,\texttt w)(y,\texttt d) \rightarrow(x,\texttt v) & \forall x,y\in\Sigma \\
(x,\texttt w)(y,\texttt z) \rightarrow(x,\texttt w) & \forall x,y\in\Sigma \\
(x,\texttt p)(y,\texttt d) \rightarrow(x,\blank)    & \forall x,y\in\Sigma \\
(x,\texttt d)(y,\blank)    \rightarrow(x,\texttt d) & \forall x,y\in\Sigma \\
(x,\texttt d)(y,\texttt t) \rightarrow(x,\texttt g) & \forall x,y\in\Sigma \\
(x,\texttt d)(y,\texttt p) \rightarrow(x,\texttt g) & \forall x,y\in\Sigma \\
(x,\texttt d)(y,\texttt v) \rightarrow(x,\texttt r) & \forall x,y\in\Sigma \\
(x,\texttt g)(y,\blank)    \rightarrow(x,\blank)    & \forall x,y\in\Sigma \\
(x,\texttt g)(y,\texttt w) \rightarrow(x,\blank)    & \forall x,y\in\Sigma \\
(x,\texttt z)(y,\blank)    \rightarrow(x,\texttt d) & \forall x,y\in\Sigma \\
(x,\texttt v)(y,\blank)    \rightarrow(x,\blank)    & \forall x,y\in\Sigma \\
(x,\texttt v)(y,\texttt d) \rightarrow(x,\blank)    & \forall x,y\in\Sigma
\end{array}
\right\}
\quad
\begin{array}{lcl}
\blank\:\texttt R    & \rightarrow & \texttt d \\
\blank\:\texttt d    & \rightarrow & \blank    \\
\blank\:\texttt g    & \rightarrow & \texttt z \\
\blank\:\texttt z    & \rightarrow & \texttt p \\
\blank\:\texttt v    & \rightarrow & \texttt v \\
\blank\:\texttt r    & \rightarrow & \blank    \\
\texttt R\:\blank    & \rightarrow & \texttt t \\
\texttt w\:\texttt d & \rightarrow & \texttt v \\
\texttt w\:\texttt z & \rightarrow & \texttt w \\
\texttt p\:\texttt d & \rightarrow & \blank    \\
\texttt d\:\blank    & \rightarrow & \texttt d \\
\texttt d\:\texttt t & \rightarrow & \texttt g \\
\texttt d\:\texttt p & \rightarrow & \texttt g \\
\texttt d\:\texttt v & \rightarrow & \texttt r \\
\texttt g\:\blank    & \rightarrow & \blank    \\
\texttt g\:\texttt w & \rightarrow & \blank    \\
\texttt z\:\blank    & \rightarrow & \texttt d \\ 
\texttt v\:\blank    & \rightarrow & \blank    \\ 
\texttt v\:\texttt d & \rightarrow & \blank    \\ 
\end{array}
\label{eq:r rules}
\end{eqnarray}
Here the new symbols $\texttt R$ and $\texttt r$ are used
to represent the beginning and end of the rightward (clockwise) rotation
respectively.
The symbols $\texttt d$ and $\texttt v$ represent the new
pulse signals being added to the system, one slow and one fast,
as explained above.
In particular,
the symbol $\texttt d$ serves as a pulsed-hold that advances one
position counterclockwise whenever the fast pulse $\texttt p$ passes through it.
The final two symbols $\texttt g$ and $\texttt z$
are used to implement the desired interaction between
$\texttt p$ and $\texttt d$.
In particular,
$\texttt g$ represents an $\texttt d$ symbol that has absorbed 
$\texttt p$ but not yet shifted from its previous position,
while $\texttt z$ represents an $\texttt d$ symbol that has absorbed
$\texttt p$ and has already shifted one step counterclockwise.
Overall,
the rule set $R_{\textrm{right}}$ implements a scheme where
the pulsed-hold symbol $\texttt d$
advances one step counterclockwise 
in each cycle (every $n$ passes or $n(n-1)$ iterations),
eventually appearing to the right of the original control position after
$n-1$ such cycles (equivalently $n(n-1)$ passes or $n(n-1)^2$ iterations).
At that point, the victory pulse $\texttt v$ is sent to freeze the
control position at the desired location.

\begin{lemma}
For the Lag system defined by $R_{\textrm{right}}$ and any string
\rm\[
(s_1,\blank)...(s_{n-2},\blank)(s_{n-1},\texttt R)(s_n,\blank)
\]\rm
$n\geq3$,
after $n(n-1)^2+n+1$ iterations
Algorithm~\ref{alg:lag} will render the memory string
\rm\[
(s_2,\blank)...(s_{n-1},\blank)(s_n,\texttt r)(s_1,\blank)
\]\em
changing the control symbol from {\rm $\texttt R$} to {\rm $\texttt r$}
and rotating the control position one step clockwise.
\label{lem:right}
\end{lemma}

\begin{proof}
We break the proof into a series of claims that will be proved separately.
It will be convenient to assume $n\geq5$ in the following.%
\footnote{
The same proof also holds for $n=4$ and $n=3$ with minor modifications.
For $n=4$, the resulting string in Claim 2 is
$
(s_{n-3},\texttt d)
(s_{n-2},\blank)
(s_{n-1},\texttt t)
(s_n,\blank)
$,
with the claims otherwise the same.
For $n=3$, the resulting string in Claim 1 is
$
(s_{n-2},\texttt d)
(s_{n-1},\texttt t)
(s_n,\blank)
$,
after which Claims 2--4 are skipped, 
and Claims 5 and 6 are the same.
}

{\bf\em Claim 1}
Starting from the initial string stated in the lemma,
after $n$ iterations 
Algorithm~\ref{alg:lag} will produce the memory string
$
(s_1,\blank)
...
(s_{n-3},\blank)
(s_{n-2},\texttt d)
(s_{n-1},\texttt t)
(s_n,\blank)
$.

{\bf\em Claim 2}
After an additional cycle (another $n(n-1)$ iterations)
the memory string will consist of
$
(s_1,\blank)
...
(s_{n-4},\blank)
(s_{n-3},\texttt d)
(s_{n-2},\blank)
(s_{n-1},\texttt t)
(s_n,\blank)
$,
hence the symbol $\texttt t$ continues to be associated at $s_{n-1}$
but the pulsed-hold token $\texttt d$ has been shifted one position to the left.

{\bf\em Claim 3}
After each of the next $n-4$ cycles,
the memory string will continue to hold $\texttt t$ at 
position $s_{n-1}$,
but the pulsed-hold symbol $\texttt d$ will be shifted one further
position to the left, with all other positions left blank.

{\bf\em Claim 4}
Therefore,
at the end of Claims 1, 2 and 3
the memory string will consist of
$
(s_1,\texttt d)
(s_2,\blank)
...
(s_{n-2},\blank)
(s_{n-1},\texttt t)
(s_n,\blank)
$,
with $\texttt t$ at position $s_{n-1}$
and the pulsed-hold token $\texttt d$ at position $s_1$.

{\bf\em Claim 5}
After one more cycle (another $n(n-1)$ iterations)
the memory string will consist of
$
(s_1,\blank)
...
(s_{n-2},\blank)
(s_{n-1},\texttt w)
(s_n,\texttt d)
$,
so $s_{n-1}$ can now see that $s_n$ should be the next control location.

{\bf\em Claim 6}
A final cycle plus an additional iteration (another $n(n-1)+1$ iterations)
will result in the final memory string
$
(s_2,\blank)
...
(s_{n-1},\blank)
(s_n,\texttt r)
(s_1,\blank)
$,
satisfying the stated outcome.
The lemma then follows since the total number of iterations taken is
$n+n(n-1)+n(n-1)(n-3)+n(n-1)+n(n-1)+1=n(n-1)^2+n+1$.

{\bf\em Proof of Claim 1:}
Starting from the initial string,
the first $n-3$ iterations 
of Algorithm~\ref{alg:lag} will
simply rotate the first pair unaltered,
according to the rule $\blank\,\blank\rightarrow\blank\,$.
The next context will be $(s_{n-2},\blank)(s_{n-1},\texttt R)$,
which causes $(s_{n-2},\texttt d)$ to be rotated by the rule
$\blank\,\texttt R\rightarrow\texttt d$.
The subsequent context will be $(s_{n-1},\texttt R)(s_n,\blank)$,
which cases $(s_{n-1},\texttt t)$ to be rotated by the rule
$\texttt R\,\blank\rightarrow\texttt t$.
One final iteration rotates the first pair $(s_n,\blank)$ unaltered,
yielding the result stated in Claim 1.
Note that the symbol $\texttt R$ can never be generated again.

{\bf\em Note}
The remaining claims all characterize how the major cycles work in the system.
Recall from the proof of Proposition~\ref{prop:p squared} that 
$\texttt p$ orbits the memory once every cycle ($n(n-1)$ iterations).
Such a periodic orbit is sustained in this new rule set.
In particular,
during each of the cycles asserted by Claims 2--5, 
the symbol $\texttt p$ will visit each memory location $s_i$ once,
in a counterclockwise order, residing for a single pass ($n-1$ iterations)
on each location,
where for the first pass $\texttt p$ will appear implicitly within a
$\texttt t$ symbol
(i.e., absorbed into a $\texttt w$, as in Proposition~\ref{prop:p squared}),
then for zero or more passes $\texttt p$ will appear as itself,
then after it encounters $\texttt d$,
$\texttt p$ will appear implicitly in a $\texttt g$ symbol
(i.e., absorbed into $\texttt d$) for a pass,
after which $\texttt p$ will appear implicitly as a $\texttt z$
(i.e., still absorbed into $\texttt d$ but shifted counterclockwise one 
position)
for another pass,
before finally appearing as itself again for any remaining passes in the cycle.

{\bf\em Proof of Claim 2:}
\emph{First pass within the cycle:}
$\texttt p$ appears implicitly as $\texttt t$ at location $s_{n-1}$.
The first $n-3$ iterations rotate the first pair unaltered.
The next context is $(s_{n-2},\texttt d)(s_{n-1},\texttt t)$,
so the pair $(s_{n-2},\texttt g)$ is rotated by rule 
$\texttt d\,\texttt t\rightarrow\texttt g$.
The next context is $(s_{n-1},\texttt t)(s_n,\blank)$,
so the pair $(s_{n-1},\texttt w)$ is rotated by rule 
$\texttt t\,\blank\rightarrow\texttt w$.
The first pass concludes with the string
$
(s_n,\blank)
(s_1,\blank)
...
(s_{n-2},\texttt g)
(s_{n-1},\texttt w)
$.

\emph{Second pass:}
$\texttt p$ appears implicitly as $\texttt g$ at location $s_{n-2}$.
The first $n-3$ iterations rotate the first pair unaltered.
The next context is $(s_{n-3},\blank)(s_{n-2},\texttt g)$,
so the pair $(s_{n-3},\texttt z)$ is rotated by rule 
$\blank\,\texttt g\rightarrow\texttt z$.
The next context is $(s_{n-2},\texttt g)(s_{n-1},\texttt w)$,
so the pair $(s_{n-2},\blank)$ is rotated by rule 
$\texttt g\,\texttt w\rightarrow\blank\,$.
The second pass concludes with the string
$
(s_{n-1},\texttt w)
(s_n,\blank)
(s_1,\blank)
...
(s_{n-3},\texttt z)
(s_{n-2},\blank)
$.

\emph{Third pass:}
$\texttt p$ appears implicitly as $\texttt z$ at location $s_{n-3}$.
The first $n-3$ iterations rotate the first pair unaltered.
The next context is $(s_{n-4},\blank)(s_{n-3},\texttt z)$,
so the pair $(s_{n-4},\texttt p)$ is rotated by rule 
$\blank\,\texttt z\rightarrow\texttt p$.
The next context is $(s_{n-3},\texttt z)(s_{n-2},\blank)$,
so the pair $(s_{n-3},\texttt d)$ is rotated by rule 
$\texttt z\,\blank\rightarrow\texttt d$.
The third pass yields the string
$
(s_{n-2},\blank)
(s_{n-1},\texttt w)
(s_n,\blank)
(s_1,\blank)
...
(s_{n-4},\texttt p)
(s_{n-3},\texttt d)
$.

\emph{Next $n-4$ passes:}
$\texttt p$ appears as itself,
followed by $\texttt d$ or a blank.
Therefore, $\texttt p$ moves one position counterclockwise after each pass.
The token $\texttt w$ at $s_{n-1}$ and $\texttt d$ at $s_{n-3}$ are also 
followed by blanks and therefore stay at the same location
by the rules $\texttt w\,\blank\rightarrow\texttt w$ and
$\texttt d\,\blank\rightarrow\texttt d$ respectively.
At the end of the next $n-4$ passes the memory string will be
$
(s_2,\blank)
...
(s_{n-3},\texttt d)
(s_{n-2},\blank)
(s_{n-1},\texttt w)
(s_n,\texttt p)
(s_1,\blank)
$.

\emph{Last ($n$th) pass:}
$\texttt p$ appears as itself at location $s_n$.
The first $n-3$ iterations rotate the first pair unaltered.
The next context is $(s_{n-1},\texttt w)(s_n,\texttt p)$,
so the pair $(s_{n-1},\texttt t)$ is rotated by rule 
$\texttt w\,\texttt p\rightarrow\texttt t$.
The next context is $(s_n,\texttt p)(s_1,\blank)$,
so the pair $(s_n,\blank)$ is rotated.
The last pass concludes with the string
$
(s_1,\blank)
...
(s_{n-4},\blank)
(s_{n-3},\texttt d)
(s_{n-2},\blank)
(s_{n-1},\texttt t)
(s_n,\blank)
$,
establishing Claim 2.

{\bf\em Proof of Claim 3:}
Consider the cycles in order $i=2,...,n-3$.
The $i$th cycle will render the following behaviour.
At the beginning of the $i$th cycle the memory string will be
$
(s_1,\blank)
...
(s_{n-i-1},\texttt d)
(s_{n-i},\blank)
...
(s_{n-1},\texttt t)
(s_n,\blank)
$.

\emph{First pass:}
$\texttt p$ appears implicitly as $\texttt t$ at location $s_{n-1}$.
The first $n-3$ iterations rotate the first pair unaltered.
The next context is $(s_{n-2},\blank)(s_{n-1},\texttt t)$,
so the pair $(s_{n-2},\texttt p)$ is rotated,
then the next context is $(s_{n-1},\texttt t)(s_n,\blank)$,
so the pair $(s_{n-1},\texttt w)$ is rotated.
The first pass concludes with the string
$
(s_n,\blank)
(s_1,\blank)
...
(s_{n-i-1},\texttt d)
(s_{n-i},\blank)
...
(s_{n-2},\texttt p)
(s_{n-1},\texttt w)
$.

\emph{Next $i-2$ passes:}
$\texttt p$ appears as itself followed by $\texttt w$ or a blank,
so after each pass $\texttt p$ moves one position counterclockwise.
The token $\texttt w$ at $s_{n-1}$ and $\texttt d$ at $s_{n-i-1}$ are also 
followed by blanks and thus stay at the same location.
Therefore, at the end of the next $i-2$ passes the memory string will be
$
(s_{n-i+2},\blank)
...
(s_{n-1},\texttt w)
(s_n,\blank)(s_1,\blank)
...
(s_{n-i-1},\texttt d)
(s_{n-i},\texttt p)
(s_{n-i+1},\blank)
$.

\emph{Next ($i$th) pass:}
$\texttt p$ appears as itself at location $s_{n-i}$.
The first $n-3$ iterations rotate the first pair unaltered.
The next context is $(s_{n-i-1},\texttt d)(s_{n-i},\texttt p)$,
so the pair $(s_{n-i-1},\texttt g)$ is rotated by rule 
$\texttt d\,\texttt p\rightarrow\texttt g$.
The next context is $(s_{n-i},\texttt p)(s_n,\blank)$,
so the pair $(s_{n-i},\blank)$ is rotated.
The $i$th pass concludes with the string
$
(s_{n-i+1},\blank)
...
(s_{n-1},\texttt w)
(s_n,\blank)
(s_1,\blank)
...
(s_{n-i-1},\texttt g)
(s_{n-i},\blank)
$.

\emph{Next ($(i+1)$st) pass:}
$\texttt p$ appears implicitly as $\texttt g$ at location $s_{n-i-1}$.
The first $n-3$ iterations rotate the first pair unaltered.
The next context is $(s_{n-i-2},\blank)(s_{n-i-1},\texttt g)$,
so the pair $(s_{n-i-2},\texttt z)$ is rotated.
The next context is $(s_{n-i-1},\texttt g)(s_{n-i},\blank)$,
so the pair $(s_{n-i-1},\blank)$ is rotated by rule 
$\texttt g\,\blank\rightarrow\blank\,$.
Therefore, at the end of the $(i+1)$st pass the memory string will be
$
(s_{n-i},\blank)
...
(s_{n-1},\texttt w)
(s_n,\blank)(s_1,\blank)
...
(s_{n-i-2},\texttt z)
(s_{n-i-1},\blank)
$.

\emph{Next ($(i+2)$nd) pass:}
$\texttt p$ appears implicitly as $\texttt z$ at location $s_{n-i-2}$.
The first $n-3$ iterations rotate the first pair unaltered.
The next context is $(s_{n-i-3},\blank)(s_{n-i-2},\texttt z)$,
so $(s_{n-i-3},\texttt p)$ is rotated.
The next context is $(s_{n-i-2},\texttt z)(s_{n-i-1},\blank)$,
so $(s_{n-i-2},\texttt d)$ is rotated.
The $(i+2)$nd pass concludes with the memory string
$
(s_{n-i-1},\blank)
...
(s_{n-1},\texttt w)
(s_n,\blank)(s_1,\blank)
...
(s_{n-i-3},\texttt p)
(s_{n-i-2},\texttt d)
$.

\emph{Next $n-i-3$ passes:}
$\texttt p$ appears as itself followed by $\texttt d$ or a blank,
so after each pass $\texttt p$ moves one position counterclockwise.
The token $\texttt w$ at $s_{n-1}$ and $\texttt d$ at $s_{n-i-2}$ are also 
followed by blanks and thus stay at the same location.
Therefore, at the end of the next $n-i-3$ passes the memory string will be
$
(s_2,\blank)
...
(s_{n-i-2},\texttt d)
...
(s_{n-1},\texttt w)
(s_n,\texttt p)
(s_1,\blank)
$.

\emph{Last ($n$th) pass:}
The same argument as the last pass of Claim 2
can be used to show the resulting string is
$
(s_1,\blank)
...
(s_{n-i-2},\texttt d)
...
(s_{n-1},\texttt t)
(s_n,\blank)
$.

Therefore, each of the cycles $i=2,...,n-3$ retains $\texttt t$ at $s_{n-1}$
while shifting $\texttt d$ one position clockwise,
establishing Claim 3.

{\bf\em Proof of Claim 4:}
Immediate from the outcome of Claim 3.

{\bf\em Proof of Claim 5:}
The cycle starts with the memory string
$
(s_1,\texttt d)
(s_2,\blank)
...
(s_{n-1},\texttt t)
(s_n,\blank)
$
obtained at the end of Claim 4.

\emph{First $n-1$ passes:}
Following the same proof as Claim 3 with $i=n-2$ for the first $i+1=n-1$ passes,
the resulting memory string will be
$
(s_2,\blank)
...
(s_{n-1},\texttt w)
(s_n,\texttt z)
(s_1,\blank)
$.

\emph{Last ($n$th) pass:}
The first $n-2$ iterations rotate the first pair unaltered.
The next context is $(s_n,\texttt z)(s_1,\blank)$,
so the pair $(s_n,\texttt d)$ is rotated,
yielding the final string
$
(s_1,\blank)
...
(s_{n-1},\texttt w)
(s_n,\texttt d)
$,
establishing Claim 5.

{\bf\em Proof of Claim 6:}
\emph{First pass:}
The first $n-2$ iterations rotate the first pair unaltered.
The next context is $(s_{n-1},\texttt w)(s_n,\texttt d)$,
so the pair $(s_{n-1},\texttt v)$ is rotated by
$\texttt w\,\texttt d\rightarrow\texttt v$,
resulting in the memory string
$
(s_n,\texttt d)
(s_1,\blank)
...
(s_{n-1},\texttt v)
$.

\emph{Next $n-2$ passes:}
Observe that the symbol $\texttt v$ will behave identically to $\texttt p$,
in that for any pass where it is followed by a blank (or $\texttt d$),
$\texttt v$ shifts one position counterclockwise according to the rules
$\blank\,\texttt v\rightarrow\texttt v$,
$\texttt v\,\blank\rightarrow\blank\,$
and
$\texttt v\,\texttt d\rightarrow\blank\,$.
Moreover, $\texttt d$ stays stationary whenever it is followed by a blank.
Therefore, for the next $n-2$ passes,
$\texttt d$ will remain stationary while $\texttt v$ shifts one position
counterclockwise,
resulting in the memory string
$
(s_2,\blank)
...
(s_n,\texttt d)
(s_1,\texttt v)
$.

\emph{Last ($n$th) pass:}
The first $n-2$ iterations rotate the first pair unaltered.
The next context is $(s_n,\texttt d)(s_1,\texttt v)$,
so the pair $(s_n,\texttt r)$ is rotated by the rule
$\texttt d\,\texttt v\rightarrow\texttt r$,
resulting in the string
$
(s_1,\texttt v)
(s_2,\blank)
...
(s_n,\texttt r)
$.

One final iteration renders the string asserted in Claim 6,
thereby establishing the claim and the lemma.
\end{proof}

Overall, 
we have demonstrated in this section that a Lag system can exert bidirectional
position control using only Algorithm~\ref{alg:lag},
where a rule set $R_{\textrm{left}}$ allows a memory access location
to be moved one position counterclockwise in $O(n)$ iterations
(Proposition~\ref{prop:left}),
and a rule set $R_{\textrm{right}}$ allows memory access location to be moved
one position clockwise in $O(n^3)$ iterations (Lemma~\ref{lem:right}).
From these two capabilities,
it becomes clear that a Turing machine can be simulated by a Lag system.

\section{Turing machines}
\label{sec:tm}

Formally,
a Turing machine $T$ consists of a tuple
$T=(Q,\Gamma,b,q_0,H,f)$, where
$Q$ is a finite set of states,
$\Gamma$ is a finite set of tape symbols,
$b\in\Gamma$ is the unique ``blank'' symbol,
$q_0\in Q$ is the unique start state,
$H\subseteq  Q\times\Gamma$ is the set of halting (state, symbol) pairs,
and $f:Q\times\Gamma\rightarrow\Gamma\times Q\times\{-1,+1\}$
is a finite set of transition rules that specify the operation of
the machine in each compute cycle.
The machine has access to a memory tape that is uni-directionally unbounded,
so memory locations can be indexed by a natural number $i\in\NN$, $i>0$,
such that there is a leftmost memory location at $i=1$ but no rightmost
memory location.

The execution of a Turing machine is defined as follows.
The tape is initialized with an input that is expressed by
a finite number of non-blank symbols, with all other locations blank,
$T$ starts in state $q_0$, and the tape head starts at 
a specified location $i_0$ (default $i_0=1$).
At the start of each compute cycle,
$T$ is in some state $q\in Q$,
the tape head is at some location $i>0$,
and a symbol $\gamma\in\Gamma$ is currently being read from the tape.
The combination $(q,\gamma)$ determines the update
$f(q,\gamma)\mapsto(\gamma', q', D)$,
specifying
that the symbol $\gamma'$ is written at the current memory location $i$,
the machine state $q$ is updated to $q'$,
and the tape head is moved to $i+D$ (i.e., one step left or right depending
on the sign of $D$).
It is assumed the machine never moves off the left end of the tape.
The compute cycle repeats until the machine encounters a configuration
$(q,\gamma)\in H$.
Non-halting computations are possible.

To facilitate subsequent proofs,
it will be helpful to understand how the computation of a Turing machine 
can be simulated using only finite memory.
A standard simulation strategy is depicted by Algorithm~\ref{alg:tm},
where a new delimiter symbol $\#\not\in\Gamma$ is used to mark off
the end of the visited memory, 
which allows additional space to be allocated when necessary.
This enables a potentially unbounded memory to be simulated
without ever having to allocate infinite storage.%
\footnote{
Note that Algorithm~\ref{alg:tm} 
requires the input to consist of at least one symbol $\gamma_1$.
To handle an empty input, one can simply set $\gamma_1=b$ (blank).
Also, Algorithm~\ref{alg:tm} does not need to explicitly track
the array length $n$, but it will be helpful to refer to in some proofs.
}

\begin{algorithm}
\caption{Finite memory simulation of a Turing machine $T=(Q,\Gamma,b,q_0,H,f)$}
\label{alg:tm}
\textbf{Input:} $\gamma_1...\gamma_{n-1}$; $i_0$
\begin{algorithmic}
\Require{$1<n$, $0<i_0<n$}
\State $m\gets\gamma_1...\gamma_{n-1}\#$\Comment{Copy input and append delimiter $\#$}
\State $n\gets\textrm{length}(m)$\Comment{Track array size}
\State $i\gets i_0$ \Comment{Initial tape head location (default $i_0=1$)}
\State $q\gets q_0$ \Comment{Initial state}
\For{k=0,1,...}
\If{$(q,m_i)\in H$}
\Comment{Halt if encounter halting pair}
\State halt
\EndIf
\State $(\gamma',q',D)\gets f(q,m_i)$
\Comment{Get update from matching transition rule}
\State $m_i\gets\gamma'$ \Comment{Overwrite current memory location}
\State $q\gets q'$ \Comment{Update state}
\State $i\gets i+D$ \Comment{Move tape head}
\If{$m_i=\#$ (hence $i=n$)}
\State $m\gets m_1...m_{n-1}b\#$\Comment{Insert new blank symbol before delimiter}
\State $n\gets n+1$\Comment{Update array size}
\EndIf
\EndFor
\end{algorithmic}
\end{algorithm}

\section{Simulating a Turing machine with a Lag system}
\label{sec:sim}

Our first main result is to show that any Turing machine can be simulated
by a restricted $(2,2)$-Lag system.
The proof will also imply that any linear bounded automaton \cite{myhill60}
can be simulated by a restricted $(2,1)$-Lag system.
Lag systems were already shown to be computationally universal in \cite{wang63},
however the original proof relies on a reduction from a lesser known form
of register machine \cite{shepherdsonsturgis63},
which proves inconvenient for our purposes.
Instead, we develop a direct reduction of Turing machines to Lag systems
that allows us to exploit the existence of small universal Turing
machines in the subsequent argument.

Given a Turing machine $T=(Q,\Gamma,b,q_0,H,f)$,
we construct a corresponding Lag system as follows.
The Lag system will use an alphabet
$\Sigma=
(\Gamma\cup\{\#\})\times(Q\cup\{\blank\})\times(\Sigma_{\textrm{left}}\cup\Sigma_{\textrm{right}})$,
where $\#\not\in\Gamma$ is a delimiter symbol,
$Q$ is the finite set of states from $T$ (such that $\blank\not\in Q$),
and $\Sigma_{\textrm{left}}$ and $\Sigma_{\textrm{right}}$ are the
position control alphabets from Section~\ref{sec:queue}.
That is, each symbol in the Lag system is a triple,
consisting of a memory symbol from 
$\Gamma_\#=\Gamma\cup\{\#\}$,
a state symbol from $Q\cup\{\blank\}$,
and a position control symbol from one of the alphabets
developed in Section~\ref{sec:queue}.

We design rules for the Lag system so that its memory string
tracks the state of the local variables in 
the Turing machine simulation, Algorithm~\ref{alg:tm}.
In particular, at the start of each iteration $k\in\NN$,
Algorithm~\ref{alg:tm} maintains a set of local variables, $m$, $n$, 
$q$ and $i$,
where $m=m_1...m_{n-1}\#$ is an array representing the current tape contents,
$n$ is the current length of $m$,
$q$ is the current state of $T$'s controller,
and $i$ is the current location of the tape head.
To mirror the values of these local variables,
the Lag system will maintain a memory string
$
s=
(m_1,\blank\,,\blank)
...
(m_i,q,\blank)
...
(m_{n-1},\blank\,,\blank)
(\#,\blank\,,\blank)
$,
such that the sequence $m_1...m_{n-1}\#$ corresponds to $m$,
the length of $s$ is $n$,
$q$ corresponds to the same controller state,
and
the location of the Turing machine's tape head, $i$, is represented by
the location of the only non-blank state symbol $q$ in the second position
of a triple.
Specifically, 
for a given Turing machine $T$,
we define the corresponding Lag system $L$ by the rule set 
determined by
Algorithm~\ref{alg:tm2lag}.%
\footnote{
Additional notes about Algorithm~\ref{alg:tm2lag}:
The rules in $L_5$ and $L_6$ can ignore the case when both 
$\gamma_1=\#$ and $\gamma_2=\#$, since the delimiter symbol can only
appear once within a string.
The rules in $L_6$ are the only ones that emit a pair of triples rather than
a single triple.
If $L_6$ is omitted, the remaining rule set can never increase the size of
the initial string.
The conditions defining the rule set $L_5$ omit the rule
$\blank\,\blank\rightarrow\blank$, but this is explicitly included as $L_1$.
}

The first main result in this paper is to establish that
Algorithm~\ref{alg:lag},
operating on the Lag system $L$ produced by Algorithm~\ref{alg:tm2lag},
simulates the execution of Algorithm~\ref{alg:tm}
for a given Turing machine $T$ on any input $\gamma_1...\gamma_{n-1}$.

\begin{theorem}
Given a Turing machine $T=(Q,\Gamma,b,q_0,H,f)$,
an input string $\gamma_1...\gamma_{n-1}$,
and a start location $i$, 
such that
$n\geq3$
and $0<i<n$,%
\footnote{
Note that this theorem requires the input to consist of at least two
symbols $\gamma_1\gamma_2$.
To handle shorter input strings, one can simply pad the input with 
blank memory symbols ($b$) as necessary.
}
let the initial string for the Lag system be
$
(\gamma_1,\blank\,,\blank)
...
(\gamma_{i-1},\blank\,,\blank)
(\gamma_i,q_0,\blank)
(\gamma_{i+1},\blank\,,\blank)
...
(\gamma_{n-1},\blank\,,\blank)
(\#,\blank\,,\blank)
$.
Then Algorithm~\ref{alg:lag} with the rule set $L$
determined by Algorithm~\ref{alg:tm2lag}
simulates the execution of Algorithm~\ref{alg:tm} on $T$ with input
$\gamma_1...\gamma_{n-1}$,
in the sense that,
for every iteration $k\in\NN$ of Algorithm~\ref{alg:tm}---%
which results in local variables $m=m_1...m_{n-1}\#$, $n$, $q$ and $i$
satisfying $0<i<n$---%
there is an iteration $t(k)\in\NN$ of Algorithm~\ref{alg:lag}
($t(k)$ monotonically increasing in $k$)
such that the memory of the Lag system is
$s=
(m_1,\blank\,,\blank)
...
(m_i,q,\blank)
...
(m_{n-1},\blank\,,\blank)
(\#,\blank\,,\blank)
$,
thus
maintaining the correspondence with the local variables
of Algorithm~\ref{alg:tm}.
\label{thm:1}
\end{theorem}

\begin{algorithm}
\caption{Reducing a Turing machine to a Lag system}
\label{alg:tm2lag}
\textbf{Input:} Turing machine $T=(Q,\Gamma,b,q_0,H,f)$
\\
\textbf{Output:} Restricted $(2,2)$-Lag system $L:\Sigma^2\rightarrow\Sigma\cup\Sigma^2$
\begin{algorithmic}

\State
\vspace*{-1.5ex}
$L_1\gets\Big\{\;
(\gamma_1,\blank\,,\blank)(\gamma_2,\blank\,,\blank)
\;\rightarrow\;
(\gamma_1,\blank\,,\blank) \;\;
\forall\gamma_1,\gamma_2\in\Gamma_\#
\;\Big\}
$
\Comment{($L_1$)}

\State
\vspace*{-0.5ex}
$L_2\gets\Big\{\;
(\gamma_1,\blank\,,\blank)(\gamma_2,q,\blank)
\;\rightarrow\;
(\gamma_1,\blank\,,\blank) \;\;
\forall\gamma_1,\gamma_2\in\Gamma_\#,\,\forall q\in Q
\;\Big\}
$
\Comment{($L_2$)}

\State
\vspace*{0ex}
$L\gets L_1\cup L_2$

\For{each transition rule $(\gamma',q',D)\gets f(q,\gamma)$ in $T$}
\rule{0em}{3ex}

\If{$D=-1$}
\State
$L\gets L\,\cup\,$\Call{LeftRules}{$q,\gamma,q',\gamma'$}
\Else{\;($D=1$)}
\State
$L\gets L\,\cup\,$\Call{RightRules}{$q,\gamma,q',\gamma'$}
\EndIf
\EndFor

\State
\Return $L$

\Statex

\Function{LeftRules}{$q,\gamma,q',\gamma'$}

\State
$L_3\gets\left\{
\begin{array}{lcll}
(\gamma\,,q\,,\blank)(\gamma_1,\blank\,,\blank)
& \rightarrow &
(\gamma',q',\texttt L) &
\forall\gamma_1\in\Gamma
\\
(\gamma_1,\blank\,,\blank)(\gamma',q',\texttt L)
& \rightarrow &
(\gamma_1,q',\blank) &
\forall\gamma_1\in\Gamma_\#
\\
(\gamma',q',\texttt L)(\gamma_1,\blank\,,\blank)
& \rightarrow &
(\gamma',\blank\,,\blank) &
\forall\gamma_1\in\Gamma_\#
\end{array}
\right\}$
\Comment{($L_3$)}

\State
\Return{$L_3$}
\EndFunction

\Statex

\Function{RightRules}{$q,\gamma,q',\gamma'$}

\State
\vspace*{-1.5ex}
$L_4\gets\Big\{\;
(\gamma\,,q\,,\blank)(\gamma_1,\blank\,,\blank)
\; \rightarrow \;
(\gamma',q',\texttt R) \;\;
\forall\gamma_1\in\Gamma_\#
\;\Big\}$
\Comment{($L_4$)}

\For{$\texttt a\texttt b\rightarrow\texttt c\in R_{\textrm{right}}$}
\rule{0em}{3ex}
\State
\vspace*{0ex}
$
L_{\texttt a\texttt b\texttt c}
\gets
\left\{
\begin{array}{lcll}
(\gamma_1,q',\texttt a)(\gamma_2,q',\texttt b)
& \rightarrow &
(\gamma_1,q',\texttt c) &
\forall\gamma_1,\gamma_2\in\Gamma_\#
\;\mbox{ if } \texttt a\neq\blank\,,\,\texttt b\neq\blank\,,\,\texttt c\neq\blank
\\
(\gamma_1,\blank\,,\blank)(\gamma_2,q',\texttt b)
& \rightarrow &
(\gamma_1,q',\texttt c) &
\forall\gamma_1,\gamma_2\in\Gamma_\#
\;\mbox{ if } \texttt a=\blank\,,\,\texttt b\neq\blank\,,\,\texttt c\neq\blank
\\
(\gamma_1,q',\texttt a)(\gamma_2,\blank\,,\blank)
& \rightarrow &
(\gamma_1,q',\texttt c) &
\forall\gamma_1,\gamma_2\in\Gamma_\#
\;\mbox{ if } \texttt a\neq\blank\,,\,\texttt b=\blank\,,\,\texttt c\neq\blank
\\
(\gamma_1,q',\texttt a)(\gamma_2,q',\texttt b)
& \rightarrow &
(\gamma_1,\blank,\blank) &
\forall\gamma_1,\gamma_2\in\Gamma_\#
\;\mbox{ if } \texttt a\neq\blank\,,\,\texttt b\neq\blank\,,\,\texttt c=\blank
\\
(\gamma_1,\blank\,,\blank)(\gamma_2,q',\texttt b)
& \rightarrow &
(\gamma_1,\blank,\blank) &
\forall\gamma_1,\gamma_2\in\Gamma_\#
\;\mbox{ if } \texttt a=\blank\,,\,\texttt b\neq\blank\,,\,\texttt c=\blank
\\
(\gamma_1,q',\texttt a)(\gamma_2,\blank\,,\blank)
& \rightarrow &
(\gamma_1,\blank,\blank) &
\forall\gamma_1,\gamma_2\in\Gamma_\#
\;\mbox{ if } \texttt a\neq\blank\,,\,\texttt b=\blank\,,\,\texttt c=\blank
\end{array}
\right\}$

\State
\vspace*{2ex}
\Comment{($L_5$)}
\EndFor

\State
\vspace*{-2ex}
$L_6\gets
\Big\{\;
(\gamma_1,q',\texttt r)(\gamma_2,\blank\,,\blank)
\; \rightarrow \;
(\gamma_1,q',\blank) \;\;
\forall\gamma_1,\gamma_2\in\Gamma_\#
\;\Big\}$
\Comment{($L_6$)}

\State
\vspace*{0ex}
$L_7\gets
\Big\{\;
(\#,q',\blank)(\gamma_1,\blank\,,\blank)
\;\rightarrow\;
(b,q',\blank)(\#,\blank\,,\blank) \;\;
\forall\gamma_1\in\Gamma
\;\Big\}$
\Comment{($L_7$)}

\State
\vspace*{0.5ex}
\Return
$L_4\cup\left(
\bigcup_{\texttt a\texttt b\rightarrow\texttt c\in R_{\textrm{right}}}
L_{\texttt a\texttt b\texttt c}
\right)\cup L_6\cup L_7$

\EndFunction

\end{algorithmic}
\end{algorithm}

\begin{proof}
The proof is by induction on the iteration $k\in\NN$ of Algorithm~\ref{alg:tm}.
The base case, when $k=0$, holds by the construction of the initialization.

Assume that at the start of iteration $k$
the correspondence holds between the local variables of Algorithm~\ref{alg:tm}
and the memory string of Algorithm~\ref{alg:lag}.
That is, if the local variables in Algorithm~\ref{alg:tm} are
$m^{(k)}=m_1...m_{n^{(k)}}\#$, $n^{(k)}$, $q^{(k)}$ and $i^{(k)}$,
with $0<i^{(k)}<n^{(k)}$,
then the memory string of Algorithm~\ref{alg:lag} will be
\begin{eqnarray}
s^{t(k)} & = &
(m_1,\blank\,,\blank)
...
(m_{i^{(k)}},q^{(k)},\blank)
...
(m_{n^{(k)}},\blank\,,\blank)
(\#,\blank\,,\blank)
.
\label{eq:ih}
\end{eqnarray}
We show that this correspondence will continue to hold at the start of the next
iteration $k+1$ of Algorithm~\ref{alg:tm}
after some finite number of additional iterations of Algorithm~\ref{alg:lag}.

Given the current state of the local variables,
the update $(\gamma',q',D)\gets f(q^{(k)},m_{i^{(k)}})$ 
of Algorithm~\ref{alg:tm} is uniquely determined.
We consider three cases depending on whether 
in the update
the tape head of the Turing machine $T$ moves left
(i.e., $D=-1$ and $1<i^{(k)}<n^{(k)}-1$),
the tape head moves right but not onto the delimiter
(i.e., $D=+1$ and $0<i^{(k)}<n^{(k)}-1$),
or the tape head moves right and onto the delimiter
(i.e., $D=+1$ and $i^{(k)}=n^{(k)}-1$).

{\bf\em Case 1: the tape head moves left:}
In this case we assume $D=-1$ and $1<i^{(k)}<n^{(k)}-1$.
(Recall that the Turing machine is not allowed to move off the left end of
the tape, so $i^{(k)}>1$.)
The new values for the local variables in Algorithm~\ref{alg:tm} 
after the $(k+1)$st iteration will become
$m^{(k+1)}=m_1...m_{i-1}\gamma'm_{i+1}...m_{n^{(k)}}\#$,
$n^{(k+1)}=n^{(k)}$,
$i^{(k+1)}=i^{(k)}-1$
and $q^{(k+1)}=q'$.

By the induction hypothesis, the memory string for the Lag system will start
out as \eqref{eq:ih}.
Consider how Algorithm~\ref{alg:lag} will update this memory string.
To de-clutter the notation, drop the superscript ${}^{(k)}$.

The initial $i-2$ triples will be rotated unmodified, by the rule $L_1$.
The next context will be $(m_{i-1},\blank\,,\blank)(m_i,q,\blank)$,
so by $L_2$ the triple $(m_{i-1},\blank\,,\blank)$ will be rotated unmodified.
The next context will be $(m_i,q,\blank)(m_{i+1},\blank\,,\blank)$,
so by the first rule in $L_3$ the first triple will be rotated as
$(\gamma',q',\texttt L)$.
Then by $L_1$ each of the next $n-2$ triples will be rotated unmodified,
yielding the string
\begin{equation}
(m_{i-1},\blank\,,\blank)
(\gamma',q',\texttt L)
(m_{i+1},\blank\,,\blank)
...
(m_{n-1},\blank\,,\blank)
(\#,\blank\,,\blank)
(m_1,\blank\,,\blank)
...
(m_{i-2},\blank\,,\blank)
.
\end{equation}
Since the next context is
$(m_{i-1},\blank\,,\blank)
(\gamma',q',\texttt L)
$,
the triple $(m_{i-1},q',\blank)$ will be rotated to the end
according to the second rule in $L_3$.
Finally, given the next context
$(\gamma',q',\texttt L)(m_{i+1},\blank\,,\blank)$,
the triple $(\gamma',\blank\,,\blank)$
will be rotated by the third rule in $L_3$,
which will yield the memory string
\begin{equation}
(m_{i+1},\blank\,,\blank)
...
(m_{n-1},\blank\,,\blank)
(\#,\blank\,,\blank)
(m_1,\blank\,,\blank)
...
(m_{i-1},q',\blank)
(\gamma',\blank\,,\blank)
.
\end{equation}
After another $n-i$ rotations using $L_1$, Algorithm~\ref{alg:lag} will obtain
\begin{equation}
(m_1,\blank\,,\blank)
...
(m_{i-1},q',\blank)
(\gamma',\blank\,,\blank)
(m_{i+1},\blank\,,\blank)
...
(m_{n-1},\blank\,,\blank)
(\#,\blank\,,\blank)
.
\end{equation}
Thus, the correspondence with the next state of the local variables
in Algorithm~\ref{alg:tm}'s simulation has been reestablished,
after a total of $i-1+1+n-2+1+1+n-i=2n$ iterations of Algorithm~\ref{alg:lag}.

{\bf\em Case 2: the tape head moves right but not onto the delimiter:}
In this case we assume $D=1$ and $0<i^{(k)}<n^{(k)}-1$.
The new values for the local variables in Algorithm~\ref{alg:tm} will become
$m^{(k+1)}=m_1...m_{i-1}\gamma'm_{i+1}...m_{n-1}\#$,
$n^{(k+1)}=n^{(k)}$,
$i^{(k+1)}=i^{(k)}+1$
and $q^{(k+1)}=q'$
at the start of the $(k+1)$st iteration.

By the induction hypothesis, the memory string for the Lag system will
start out as \eqref{eq:ih}.
Consider how Algorithm~\ref{alg:lag} updates this memory string.
Once again, drop the superscript ${}^{(k)}$ to de-clutter the notation. 

The initial $i-1$ triples will be rotated unmodified to the end,
by the rules $L_1$ and $L_2$. 
The next context will be $(m_i,q,\blank)(m_{i+1},\blank\,,\blank)$,
so by $L_4$ the first triple will be rotated as $(\gamma',q',\texttt R)$.
The next triple will be rotated unmodified, by $L_1$,
resulting in the string
\begin{equation}
(m_{i+2},\blank\,,\blank)
...
(m_{n},\blank\,,\blank)
(\#,\blank\,,\blank)
(m_1,\blank\,,\blank)
...
(m_{i-1},\blank\,,\blank)
(\gamma',q',\texttt R)
(m_{i+1},\blank\,,\blank)
.
\label{eq:R}
\end{equation}

At this point, we make three key observations.
\emph{First:}
each triple in \eqref{eq:R} has a blank second position if and only if the
third position is also blank.
As long as this property holds, the only rules that can match a
context are in $L_1\cup L_5\cup L_6$.%
\footnote{
Although the second and third rules in $L_3$ also satisfy this property,
these rules can be ignored in this case,
since the position control symbol $\texttt L$ cannot be generated from the
string \eqref{eq:R} before the case concludes.
}
\emph{Second:}
$q'$ is the only non-blank symbol that can appear in the second position of
any triple for the remainder of this case,
since it is the only non-blank symbol in the second position of any triple
in \eqref{eq:R}
and the only non-blank symbol that can be generated in the second position
by any rule in $L_1\cup L_5\cup L_6\cup L_7$.
\emph{Third:}
the set of rules $L_{\textrm{right}}=L_1\cup L_5$
corresponds to the position control rules in $R_{\textrm{right}}$,
and every rule in  $L_{\textrm{right}}$ either generates a blank in both the
second and third position,
or a non-blank in both the second ($q'$) and third position,
thus maintaining the first observation.

From these properties we conclude that,
until the position symbol $\texttt r$ is generated and appears in a context,
the only applicable rules are in $L_{\textrm{right}}$.
Therefore,
the third position of any triple generated by Algorithm~\ref{alg:lag}
will behave identically to Lemma~\ref{lem:right}.
This implies that after another $n(n-1)^2+n+1$ iterations
the memory string will have a position control sequence that matches
Lemma~\ref{lem:right}'s conclusion;
namely,
\begin{equation}
\hspace*{-0.75em}
(m_{i+3},\blank\,,\blank)
...
(m_{n},\blank\,,\blank)
(\#,\blank\,,\blank)
(m_1,\blank\,,\blank)
...
(m_{i-1},\blank\,,\blank)
(\gamma',\blank\,,\blank)
(m_{i+1},q',\texttt r)
(m_{i+2},\blank\,,\blank)
.
\label{eq:right penultimate}
\end{equation}

The next $n-2$ iterations will rotate the first triple unaltered.
The subsequent context will be
$(m_{i+1},q',\texttt r)(m_{i+2},\blank\,,\blank)$,
so the triple $(m_{i+1},q',\blank)$ will be rotated on the next iteration,
by $L_6$.
Then for the next $n-i-1$ iterations the first triple will be rotated unaltered,
yielding
\begin{equation}
(m_1,\blank\,,\blank)
...
(m_{i-1},\blank\,,\blank)
(\gamma',\blank\,,\blank)
(m_{i+1},q',\blank)
...
(m_{n},\blank\,,\blank)
(\#,\blank\,,\blank)
.
\label{eq:right ultimate}
\end{equation}
Thus, the correspondence with the next state of the local variables
in Algorithm~\ref{alg:tm}'s simulation has been reestablished,
after a total of
$i-1 + 1 + 1 + n(n-1)^2+n+1 + n-2 + 1 + n-i-1
=n(n-1)^2+3n$ iterations.

{\bf\em Case 3: the tape head moves right and onto the delimiter:}
In this case we assume $D=1$ and $i^{(k)}=n^{(k)}-1$.
The new values for the local variables in Algorithm~\ref{alg:tm} 
will therefore become
$m^{(k+1)}=m_1...m_{n-2}\gamma'b\#$,
$n^{(k+1)}=n^{(k)}+1$,
$i^{(k+1)}=i^{(k)}+1$
and $q^{(k+1)}=q'$
at the start of the $(k+1)$st iteration;
that is, Algorithm~\ref{alg:tm} inserts a new blank symbol $b$
at memory location $n^{(k)}$ and expands the length of its memory array by one.

By the induction hypothesis, the memory string for the Lag system will
start out as
\begin{eqnarray}
s^{t(k)} & = &
(m_1,\blank\,,\blank)
...
(m_{n^{(k)}-2},\blank\,,\blank)
(m_{n^{(k)}-1},q,\blank)
(\#,\blank\,,\blank)
.
\label{eq:ih3}
\end{eqnarray}
Consider how Algorithm~\ref{alg:lag} will update this string.
Once again, drop the superscript ${}^{(k)}$ to de-clutter the notation. 

Initially this case proceeds isomorphically to the previous case,
until the position control symbol $\texttt r$ appears and is
subsequently replaced, 
meaning that after $n(n-1)^2+3n$ iterations
Algorithm~\ref{alg:lag} will produce the memory string
\begin{equation}
(m_1,\blank\,,\blank)
...
(m_{n-2},\blank\,,\blank)
(\gamma',\blank\,,\blank)
(\#,q',\blank)
.
\end{equation}
Now,
in this situation,
the state symbol $q'$ has become associated with the delimiter $\#$,
so we continue iterating.
From this string, 
another $n-1$ iterations will rotate the first triple unmodified to the end,
resulting in the context $(\#,q',\blank)(m_1,\blank\,,\blank)$.
At this point, $L_7$ is applied to expand the size of
Algorithm~\ref{alg:lag}'s memory string by one,
since after deleting the first triple $(\#,q',\blank)$,
the pair of triples $(b,q',\blank)(\#,\blank\,,\blank)$
is appended to the end of the string, resulting in
\begin{equation}
(m_1,\blank\,,\blank)
...
(m_{n-2},\blank\,,\blank)
(\gamma',\blank\,,\blank)
(b,q',\blank)
(\#,\blank\,,\blank)
.
\end{equation}
Thus, the correspondence with the next state of the local variables
in Algorithm~\ref{alg:tm}'s simulation has been reestablished,
after a total of 
$n(n-1)^2+3n + n-1 + 1
=n(n-1)^2+4n$ iterations.

Since the induction step holds in each case, the theorem has been proved.
\end{proof}

There are a few interesting corollaries that follow from this result.
First, Theorem~\ref{thm:1} shows that any Turing machine
can be simulated by a restricted $(2,2)$-Lag system,
since all of the rules in Algorithm~\ref{alg:tm2lag} use a context
of size $2$ and emit either one or two output symbols
(here each symbol is a triple).

\begin{corollary}
Any linear bounded automaton can be simulated,
in the same sense as Theorem~\ref{thm:1},
by a restricted $(2,1)$-Lag system.
\label{cor:1}
\end{corollary}

\begin{proof}
A linear bounded automaton is simply a Turing machine $T$ that never moves off
the left or right of its input \cite{myhill60,sipser13}.
For such a $T$, given an input string $\gamma_1...\gamma_{n-1}$,
Algorithm~\ref{alg:tm} always maintains the index variable $i$ between $0<i<n$.
From the proof of Theorem~\ref{thm:1},
one can see that the corresponding Lag system $L$ produced by
Algorithm~\ref{alg:tm2lag}
will maintain an exact simulation of $T$
without ever applying $L_7$,
which is the only rule that emits more than one output triple.
Hence, $L_7$ can be dropped from the simulation,
and the remaining rule set defines a restricted $(2,1)$-Lag system.
\end{proof}

Note that a restricted $(N,1)$-Lag system can be simulated by
generalized autoregressive (greedy) decoding of an $(N+1)$-gram model,
which implies that a $3$-gram model can also simulate any linear bounded
automaton.
However, these models are insufficient for simulating an arbitrary
Turing machine.

\begin{corollary}
Any restricted $(N,1)$-Lag system is not Turing universal,
hence, generalized autoregressive (greedy) decoding of any $K$-gram model
also cannot be Turing universal.
\label{cor:2}
\end{corollary}

\begin{proof}
This follows directly from the fact that a restricted $(N,1)$-Lag system
can never change the length of its memory string after initialization,
since every iteration deletes and appends exactly one symbol,
hence its computation can be simulated by a linear bounded automaton.
Linear bounded automata cannot simulate every Turing machine,
since every language recognizable by a linear bounded automaton is
context sensitive \cite{landweber63},
yet there are recursively enumerable languages that are not context
sensitive \cite{chomsky59,salomaa73}.
\end{proof}

\begin{corollary}
Any Turing machine can be simulated by extended autoregressive (greedy)
decoding of a $K$-gram model with $K\geq4$.
\end{corollary}

\begin{proof}
Follows directly from the proof of Theorem~\ref{thm:1},
since every rule produced by Algorithm~\ref{alg:tm2lag}
is expressible by a $3$-gram ($L_1$--$L_6$) or a $4$-gram ($L_7$),
where $L_7$ is necessary for Turing completeness and requires extended
autoregressive decoding to realize.
Note that simulating the application of $L_7$ requires
an extended autoregressive decoding strategy
(as discussed in Section~\ref{sec:auto})
to generate an output bigram given a bigram context.
\end{proof}

\section{A universal Lag system}
\label{sec:utm}

The primary goal of this paper is to demonstrate that current language models
are computationally universal under extended autoregressive decoding.
For any given language model,
the most direct method for establishing such a result is to identify a known
computationally universal system that the model can simulate.

Ultimately, any assertion of computational universality relies on the
Church-Turing thesis, that all computational mechanisms are expressible
by a Turing machine \cite{sipser13,mooremertens11}.
The concept of a \emph{universal Turing machine}---%
a Turing machine $U$ that can simulate the execution of any other Turing
machine $T$ on any input---%
was developed by Alan Turing to solve the \emph{Entscheidungsproblem}
\cite{turing37}.
Proving computational universality of a system therefore reduces to
establishing that the system can simulate the operation of a universal
Turing machine.

Given the similarity between autoregressive decoding of a language model and
the updates of a Lag system, as established in Proposition~\ref{prop:auto2lag},
it is natural to seek a universal Lag system to provide the foundation for a
proof of universality.
The results of the previous section (Theorem~\ref{thm:1})
now make the construction of a universal Lag system straightforward.

Let $L(T)$ denote the Lag system that is produced by applying
Algorithm~\ref{alg:tm2lag} to a Turing machine $T$.

\begin{theorem}
For any universal Turing machine $U$,
the Lag system $L(U)$ obtained by applying Algorithm~\ref{alg:tm2lag} to $U$
is also universal, in the sense that
Algorithm~\ref{alg:lag} operating on $L(U)$, like $U$, is able to simulate
the execution of any Turing machine $T$ on any input
$\gamma=\gamma_1...\gamma_{n-1}$; $i_0$.
\hspace*{-2em}
\label{thm:2}
\end{theorem}

\begin{proof}
The universality of $U$ implies that for any Turing machine $T$ there exist
two computable functions---%
an encoding function $c_T$ and a location function $\ell_T$---%
such that, 
for any input $\gamma$ and initial location $i_0$,
the execution of $U$ on $c_T(\gamma)$ starting from tape location $\ell_T(i_0)$
simulates the execution of $T$ on $\gamma$ starting from $i_0$.
By Theorem~\ref{thm:1} it then immediately follows that,
using $L(U)$,
Algorithm~\ref{alg:lag}
is able to simulate the execution of $U$ on
$c_T(\gamma)$ starting from location $\ell_T(i_0)$,
hence it is also able to simulate
the execution of $T$ on $\gamma$ starting from $i_0$.
\end{proof}

The concluding
argument below will be based on formulating then simulating a universal
Lag system $L(U)$ with extended autoregressive decoding of a language model.
A key challenge will be to identify a compact universal Lag system,
since the entire prompt design for a large language model needs to
fit within the model's bounded context window.
To that end, we consider a universal Lag system derived from
a small universal Turing machine.

There has been a long running effort to identify the smallest universal Turing
machines in terms of the number of states and tape symbols used,
starting with \cite{shannon56}.
A gap remains between the known upper and lower bounds on
the state and symbol counts for a universal Turing machine
\cite{neary08,nearywoods09,nearywoods12},
but progressively smaller universal Turing machines have been identified.
In this paper, we will consider one such machine, $U_{15,2}$,
that uses only 15 states and 2 tape symbols \cite{nearywoods09},
which is Pareto optimal in terms of the smallest known
universal Turing machines \cite{neary08}.
Formally, the Turing machine $U_{15,2}$ can be defined by a tuple
$U_{15,2}=(Q,\Gamma,b,q_0,H,f)$, where
$Q=\{A,B,C,D,E,F,G,H,I,J,K,L,M,N,O\}$,
$\Gamma=\{0,1\}$,
$b=0$,
$q_0=A$,
$H=\{(J,1)\}$,
and the transition function $f$ is defined in Table~\ref{tab:U15,2}.
The proof of universality of $U_{15,2}$ also requires the
initial position of the tape head, $i_0$, to depend
on the target system being simulated \cite{nearywoods09}.

\begin{table}[t]
\centering
\begin{tabular}{c|cccccccc}
    & $A$     & $B$     & $C$     & $D$     & $E$     & $F$     & $G$     & $H$     \\ \hline
$0$ & $0,+,B$ & $1,+,C$ & $0,-,G$ & $0,-,F$ & $1,+,A$ & $1,-,D$ & $0,-,H$ & $1,-,I$ \\
$1$ & $1,+,A$ & $1,+,A$ & $0,-,E$ & $1,-,E$ & $1,-,D$ & $1,-,D$ & $1,-,G$ & $1,-,G$ \\ \hline
\end{tabular}

\begin{tabular}{c|ccccccc}
    & $I$     & $J$     & $K$     & $L$     & $M$     & $N$     & $O$     \\ \hline
$0$ & $0,+,A$ & $1,-,K$ & $0,+,L$ & $0,+,M$ & $0,-,B$ & $0,-,C$ & $0,+,N$ \\
$1$ & $1,-,J$ & halt    & $1,+,N$ & $1,+,L$ & $1,+,L$ & $0,+,O$ & $1,+,N$ \\ \hline
\end{tabular}
\caption{
Transition table for the universal Turing machine $U_{15,2}$.
Rows are indexed by the read symbol $\gamma$,
columns are indexed by the state $q$,
and each table entry $(\gamma',D,q')$ specifies the write symbol $\gamma'$,
the tape head move $D\in\{-1,+1\}$, and the next state $q'$.
The initial state is $A$.
}
\label{tab:U15,2}
\end{table}

Given its compact size,
we therefore adopt the universal Lag system $L(U_{15,2})$
as the target system to be simulated for the remainder of this paper.

\begin{corollary}
Algorithm~\ref{alg:lag} operating on $L(U_{15,2})$ is able to simulate
the execution of any Turing machine $T$ on any input
$\gamma=\gamma_1...\gamma_{n-1}$; $i_0$.
\end{corollary}

\begin{proof}
Immediate from the universality of $U_{15,2}$ \cite{nearywoods09}
and
by
Theorem~\ref{thm:2}.
\end{proof}

The universal Lag system $L(U_{15,2})$
consists of $2027$ rules defined over $262$ ``symbols'',
where each symbol is a triple, as explained in Section~\ref{sec:sim}.
Of the $2027$ rules, only $16$ produce a pair of output symbols for the
matching context (i.e., $L_7$),
while the remaining $2011$ only output a single symbol for the matching context
(i.e., $L_1$--$L_6$).

\section{Simulating a universal Lag system with a language model}
\label{sec:gemini}

Our final task is to demonstrate that an existing large language model can
simulate the execution of the universal Lag system $L(U_{15,2})$
on any input string.
This will be achieved by developing a specific prompt for the language model
that drives extended autoregressive (greedy) decoding to mimic the behaviour
of $L(U_{15,2})$.

As a preliminary step, note that the $262$ symbols used by 
$L(U_{15,2})$ are not in a convenient format for this purpose,
as each symbol is a triple that requires auxiliary characters to describe
(such as parentheses and commas).
Rather than have the language model 
generate a bunch of extraneous syntactic detail,
we streamline the task by introducing a simple invertible mapping 
between the $262$ triples and $262$ token pairs.
In particular, we leverage a simple bijection
where each triple is assigned a unique token pair,
implementable by two dictionaries:
an encoding dictionary that maps triples to token pairs,
and an inverse decoding dictionary that maps token pairs to triples.
%
The computation then proceeds by
encoding the Turing machine simulation as a sequence of token pairs,
rather than a sequence of triples.
(Note that realizing this simulation requires extended autoregressive decoding,
as introduced in Section~\ref{sec:auto}, since the language model must
generate 2 or 4 tokens for each given context.)
The resulting formulation simulates the execution of the universal Turing
machine by incorporating this bijection in the computable function
$c_{U_{15,2}}$ used in the proof of Theorem~\ref{thm:2}.

We develop a prompting strategy that consists of two components:
a ``system prompt'' that provides the full rule set 
(expressed over the token pair encoding)
to the language model,
and a ``sliding window prompt''
where the next symbol pair (4 tokens) from the input sequence is appended.
(This prompt structure
is reminiscent of the recent streaming model proposed by \cite{xiaoetal24}.)
On each iteration,
the next symbol pair (4 tokens) from the sequence is appended to the system
prompt and given to the language model as input;
the output of the language model (2 or 4 tokens)
is then appended to the end of the
sequence, as shown in Figure~\ref{fig:uprompt}.
To ensure a deterministic system, we use a temperature of zero and fix
all the random seeds that define the language model's behaviour,
as discussed in Section~\ref{sec:auto}.
To allow the language model to emit a variable number of tokens
for each context window,
we employ extended autoregressive decoding as explained in
Section~\ref{sec:auto}
where an implicit latent halt token $h$ outside the base alphabet of
$262$ token pairs is used.

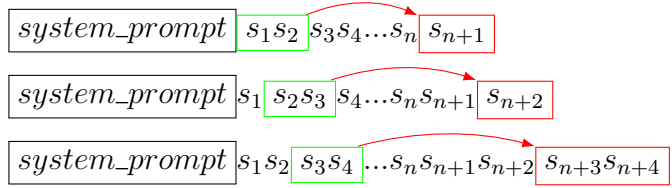
\begin{figure}
\centering
\hskip2cm
\begin{tikzpicture}[baseline=1]
\draw[-{Latex},red] (3.950,2.050) .. controls (4.447,2.261) and (4.943,2.261) .. (5.440,2.050);
\put (0,50) {
\fcolorbox{black}{white}{$system\rule{0.4em}{0.1ex}prompt$}\fcolorbox{green}{white}{$s_1s_2$}$s_3s_4...s_n\fcolorbox{red}{white}{$s_{n+1}$}$ };
\draw[-{Latex},red] (4.290,1.189) .. controls (4.933,1.4) and (5.577,1.4) .. (6.220,1.189);
\put (0,25) {
\fcolorbox{black}{white}{$system\rule{0.4em}{0.1ex}prompt$}$s_1$\fcolorbox{green}{white}{$s_2s_3$}$s_4...s_ns_{n+1}\fcolorbox{red}{white}{$s_{n+2}$}$ };
\draw[-{Latex},red] (4.650,0.309) .. controls (5.433,0.52) and (6.217,0.52) .. (7.000,0.309);
\put (0,0) {
\fcolorbox{black}{white}{$system\rule{0.4em}{0.1ex}prompt$}$s_1s_2$\fcolorbox{green}{white}{$s_3s_4$}$...s_ns_{n+1}s_{n+2}\fcolorbox{red}{white}{$s_{n+3}s_{n+4}$}$ };
\end{tikzpicture}
\caption{
\emph{Prompting strategy} for simulating a universal Lag system.
The $system\rule{0.4em}{0.1ex}prompt$ consists of copies of the entire rule set,
which is prepended to the prompt used in each call to the language model.
A sliding context window of size $2$ is moved through the symbol sequence,
emitting $1$ or $2$ symbols that are appended to the end of the
sequence in each iteration.
(Note that each symbol is a pair of tokens.)
The sliding context window advances $1$ position per iteration.
For example, in the first row, 
the prompt given to the language model on the initial iteration
is $system\rule{0.4em}{0.1ex}prompt\: s_1s_2$;
for the next iteration,
the prompt given to the language model is $system\rule{0.4em}{0.1ex}prompt\: s_2s_3$,
then $system\rule{0.4em}{0.1ex}prompt\: s_3s_4$ in the subsequent iteration, and so on.
}
\label{fig:uprompt}
\end{figure}

Finally, we choose a specific large language model to verify whether
extended autoregressive (greedy) decoding is indeed able to replicate
the behaviour of $L(U_{15,2})$.
For this purpose, we use 
{\tt gemini-1.5-pro-001},
which is a publicly released large language model
that can be accessed at
\url{https://ai.google.dev/gemini-api/docs/models/gemini}.
After some experimentation, we developed a single system prompt that
drives this model to correctly execute each of the $2027$ rules.
We refer to this system prompt as $S_{\textrm{gemini}}$.
Based on these developments, we reach the final claim.

\begin{theorem}
For any string $\gamma_1...\gamma_{n-1}$,
{\tt gemini-1.5-pro-001}
with extended autoregressive (greedy) decoding,
using the system prompt string $S_{\textrm{gemini}}$ and a sliding
window of length $2$,
is able to 
simulate
the execution of the universal Lag system
$L(U_{15,2})$ on $\gamma_1...\gamma_{n-1}\#$.
\end{theorem}

\begin{proof}
The proof consists of an enumeration of the $2027$ cases,
where, for each rule $x_1x_2\rightarrow y$ in $L(U_{15,2})$,
the prompt 
$S_{\textrm{gemini}}x_1x_2$ 
is provided to the
language model and the model's response is verified to correspond
to the rule's output $y$.
\end{proof}

From this theorem, by the Church-Turing thesis, we conclude that
{\tt gemini-1.5-pro-001}
under extended autoregressive (greedy) decoding
is a general purpose computer.
Notably, no additional computational mechanisms 
beyond extended autoregressive decoding are required to achieve this result.

\section{Related work}
\label{sec:related}

It has already been intuited 
in the literature
that large language models are able to express
arbitrary computation via autoregressive decoding \cite{jojicetal23}---%
an informal conclusion reached through a series of prompt designs that
elicit increasingly sophisticated computational behaviour.
However, a proof has not been previously demonstrated.
By establishing such a result in this paper, 
we observe more precisely that generalized autoregressive decoding 
(Section~\ref{sec:auto}) is only sufficient for realizing linear bounded
automata (Corollary~\ref{cor:1}),
whereas extended autoregressive decoding is required for simulating
general Turing machines (Theorem~\ref{thm:1} and Corollary~\ref{cor:2}).
Previous work \cite{schuurmans23} has demonstrated that augmenting a
language model with an external random access memory allows
universal computational behaviour to be elicited via prompting,
but such a result is weakened by the need to introduce external regular
expression parsers for managing the interaction between the language model
and the memory.
Here we do not consider any such augmentations of the model,
and instead only consider extended autoregressive decoding introduced in
Section~\ref{sec:auto}.
That is,
even though variety of other extensions to transformer-based language models
continue to be investigated \cite{mialonetal23},
these are unnecessary for achieving computational universality
(although they undoubtedly have practical benefits).

A significant amount of prior work has investigated the hypothetical
expressiveness of different architectures for sequential neural networks,
asking whether a given architecture is able, in principle,
to express a given model of computation
(e.g., Turing machine, formal language class, or circuit family).
Such works typically argue for the existence of model parameters that realize
targeted behaviour,
but do not account for the learnability of such parameters from data.
Moreover, such results generally do not apply directly to pre-trained
models where the parameters have been frozen.

Regarding Turing machine simulation,
an extensive literature has investigated the ability of various
sequential neural network architectures to express universal computation,
focusing, for example, on
recurrent neural networks \cite{siegelmannsontag92,chenetal18},
transformers \cite{perezetal19,bhattamishraetal20},
attention mechanisms in particular \cite{perezetal21},
transformers with recurrence (informally) \cite{dehghanietal19},
and neural Turing machines
\cite{perezetal19}.
A related line of research has investigated novel computational architectures
tailored to transformers that retain universal computational properties
\cite{giannouetal23}.
A common shortcoming of these studies is that,
by focusing on representational capacity, 
they exploit high precision weights to encode data structures,
such as multiple stacks,
without considering whether these specific weight configurations are
learnable from data.
More recently, there has been some attempt to control the precision of the
weights and still approximate general computational behaviour 
with reasonable statistical complexity \cite{weietal22a}.
Another interesting result is that Turing completeness can still be
achieved by a recurrent neural network with bounded precision neurons,
provided the hidden state can grow dynamically \cite{chungsiegelmann21}.
These prior results 
nevertheless
do not apply to existing large language models without
altering their weights (as far as anyone currently knows).

The literature on neural Turing machines,
although directly inspired by theoretical models of computation, 
has remained largely empirical
\cite{gravesetal14,gravesetal16,kaisersutskever16,kurachetal16}.
Some of this work has shown that augmenting recurrent neural
networks with differentiable memory 
in the form of stack or queue
does allow rich language structures to be learned from data
\cite{grefenstetteetal15}.
The computational universality of such models has only been
demonstrated 
theoretically
by \cite{perezetal19}, with the same shortcomings as above.

Another line of work has considered practical restrictions on neural network
parameterizations, focusing on the effects of limited precision representation.
Again, these works focus on hypothetical expressiveness rather than 
learnability or applicability to pre-trained models.
Useful insights have been obtained regarding the ability of
alternative architectures to realize restricted forms of computation with
finite precision parameterizations,
including the expressiveness of
recurrent neural networks (including LSTMs and GRUs)
\cite{weissetal18,korskyberwick19},
attention mechanisms with depth and width bounds \cite{hahn20},
linear transformers \cite{irieetal23},
transformers \cite{weissetal21,fengetal23,merrillsabharwal24,lietal24},
and state space models 
\linebreak
\cite{merrilletal24}.
None of these results apply to existing, pre-trained language models,
nor do they establish Turing completeness, but they do identify
clear limitations of myopic next token or bounded chain of thought processing
with finite precision representations.

In this paper, we have primarily focused on the ability of language models to
simulate general Turing machines or linear bounded automata.
There is an extensive literature that investigates the relationship between
various models of computation based on automata and alternative grammar classes
\cite{chomsky59,chomskyschutzenberger63}.
The relationship between linear bounded automata \cite{myhill60}
and context sensitive grammars \cite{chomsky59} is particularly interesting.
Initially it was established by \cite{landweber63} that non-deterministic
linear bounded automata can only recognize context sensitive languages 
\cite{chomsky59},
with their full equivalence subsequently established by \cite{kuroda64}.
Remarkably, the question of whether 
deterministic versus non-deterministic linear bounded automata 
exhibit the same language recognition ability
(and thus equivalent to context sensitive grammars)
has remained open for 60 years \cite{kutribetal18},
even though it is easy to show a corresponding equivalence for Turing machines
\cite{sipser13}.
The results for Lag systems developed in Section~\ref{sec:queue}
were inspired by the finding that any context sensitive grammar
can be converted to a purely one-sided or direction-independent context
sensitive grammar \cite{penttonen74,kleijnetal84}.
The earlier Tag systems studied by \cite{post43}
also inspired a long sequence of works that establish
Turing completeness of nearby variants \cite{wang63,cockeminsky64},
ultimately leading to the key findings that 
one dimensional cellular automata are Turing complete \cite{cook04},
and later to the identification of the smallest known universal Turing machines
\cite{nearywoods09,nearywoods12}.
We have directly exploited the results of the latter work in this paper,
but it remains open whether early variants of Tag systems might lead to more
compact reductions and simpler proofs of universality.

\section{Conclusion}
\label{sec:concl}

We have shown that a particular large language model,
{\tt gemini-1.5-pro-001},
is Turing complete
under extended autoregressive (greedy) decoding.
This result is not hypothetical:
the proof establishes Turing completeness of the current model in its
publicly released state.
That is,
{\tt gemini-1.5-pro-001}
is already a general purpose computer,
not merely a hypothetical computer.
We expect that the same is true of almost all currently released
large language models, both closed and open source,
but we have not yet conducted the necessary verifications on other alternatives.
Performing such a verification likely requires re-engineering of the system
prompt and possibly a fair amount of experimentation in each case,
but the proof strategy remains the same:
verify that each of the $2027$ production rules in the
universal Lag system $L(U_{15,2})$ are correctly simulated.
We do not know whether a single system prompt can be successful for a large set
of language models, but that too is a possibility worth considering.

It remains future work to consider alternative reductions
from a universal model of computation (universal Turing machine,
Lag system, or some other alternative)
that leads to a more compact verification.
The $2027$ rules in the current proof is a significant reduction from our 
earlier attempts, and it seems like even this rule set can likely be 
compressed further,
leading to easier verifications for other language models.

There are some interesting implications from these findings.
Essentially, we have established that a current large language model is already
able to simulate any algorithm or computational mechanism,
by the Church-Turing thesis.
However,
the key benefit of a large language model over a classical (formal) computer
is that, while a language model sacrifices nothing in terms of ultimate
computational power, it offers a far superior interface for human users 
to express their computational intent.
In principle, a human user can bypass full formalization their problem,
and can even avoid providing precise step by step instructions for how to solve
infinitely many variations of their particular problem instance
(i.e., design an algorithm or write a program),
yet the user can still hope to elicit useful computational behaviour
toward generating a solution.
The consequences of providing a more natural interface for producing desired
computational behaviour have already become widely apparent.

Also, as is the case with classical computers,
long computational sequences are necessary for solving arbitrary tasks.
Although it is not yet the most common mode of operation for language models,
processing very large inputs and emitting extremely long output sequences
is inevitable for tackling serious use cases.

Finally, we conclude with an (admittedly extreme) conjecture:
training to perform next token prediction on human generated text
does not ``teach'' a large language model to reason or compute;
rather, with probability nearly $1$ a randomly initialized language model
is already a universal computer,
and next token prediction is merely a mechanism 
for making its input-output behaviour understandable to humans.
This hypothesis is falsifiable in principle, but we have not yet been able
to prove it either way.

\section*{Acknowledgments}

Sincere thanks to Petros Maniatis, Csaba Szepesv\'{a}ri, Alex Lewandowski,
Benjamin Van~Roy and Doina Precup for essential discussions and
comments on this paper.
Support from the CIFAR Canada AI Research Chairs Program,
NSERC, and Amii are also gratefully acknowledged.

\bibliographystyle{apalike}

\end{document}